\documentclass{jmlr}

\usepackage{longtable}% for long tables

\usepackage{booktabs}
\usepackage{wrapfig}
\usepackage{color}

\usepackage{amsfonts}
\usepackage{algorithmic}
\usepackage{algorithm}

\newcommand{\argmax}{\operatorname*{argmax}} %

\newcommand{\var}{\operatorname*{Var}}

\newcommand{\kl}{\operatorname*{KL}}

\newcommand{\E}{\mathbb{E}}

\newcommand{\R}{\mathbb{R}}

\newcommand{\sep}{|}

\newcommand{\s}{\mathcal{S}}
\newcommand{\A}{\mathcal{A}}

\newcommand{\gc}{\mathcal{G}}
\newcommand{\hc}{\mathcal{H}}

\newcommand{\un}{\mathbf{1}}

\makeatletter
\let\Ginclude@graphics\@org@Ginclude@graphics 
\makeatother

\title[Geometric Value Iteration]{Geometric Value Iteration: Dynamic Error-Aware KL Regularization for Reinforcement Learning}

\author{\Name{Toshinori Kitamura} \Email{kitamura.toshinori.kt6@is.naist.jp}\\
\Name{Lingwei Zhu} \Email{zhu.lingwei.zj5@is.naist.jp}\\
\Name{Takamitsu Matsubara} \Email{takam-m@is.naist.jp}\\
\addr Nara Institute of Science and Technology, Nara, JAPAN}
 
\begin{document}

\maketitle

\begin{abstract}
   The recent boom in the literature on entropy-regularized reinforcement learning (RL) approaches reveals that Kullback-Leibler (KL) regularization brings advantages to RL algorithms by canceling out errors under mild assumptions.
   However, existing analyses focus on fixed regularization with a constant weighting coefficient and do not consider cases where the coefficient is allowed to change dynamically.
   In this paper, we study the dynamic coefficient scheme and present the first asymptotic error bound.
   Based on the dynamic coefficient error bound, we propose an effective scheme to tune the coefficient according to the magnitude of error in favor of more robust learning.
   Complementing this development, we propose a novel algorithm, Geometric Value Iteration (GVI), that features a dynamic {\it error-aware} KL coefficient design with the aim of mitigating the impact of errors on performance.
   Our experiments demonstrate that GVI can effectively exploit the trade-off between learning speed and robustness over uniform averaging of a constant KL coefficient.
   The combination of GVI and deep networks shows stable learning behavior even in the absence of a target network, where algorithms with a constant KL coefficient would greatly oscillate or even fail to converge.
\end{abstract}

\section{Introduction}\label{sec:introduction}

The recently impressive successes of reinforcement learning (RL) rely heavily on the use of nonlinear function approximators such as deep networks \citep{mnih2015human,Silver-nature2017}.
However, the power of nonlinear approximators comes at a cost that approximation or estimation errors can easily go uncontrolled due to stochastic approximation using noisy samples \citep{fu19a}, leading to performance oscillation or even divergent learning \citep{lillicrap2015continuous,fujimoto2018addressing}.
While error propagation has been studied in detail in the literature of approximate dynamic programming (ADP) methods \citep{Munos08-finiteTimeAVI,scherrer2015approximate}, little is understood in the case of nonlinear approximation, such as whether the analyses in ADP still hold true.
In practice, several empirical tricks like target networks or asynchronous updates need to be used to ensure stability and convergence for learning with deep networks \citep{mnih2015human,pmlr-v80-haarnoja18b}.

The recent boom in the literature on entropy-regularized RL highlights the use of Kullback-Leibler (KL) divergence as a regularization term in the reward \citep{azar2012dynamic,kozunoCVI,vieillard2020leverage}.
It is known that by adding KL regularization, errors are \emph{grouped} in the sense that they are accumulated as a summation (more details in Section \ref{sec:notations}).
In standard $L_{p}$ norm error propagation analysis, this summation is within the norm, as compared to the summation-over-norm of the standard ADP results \citep{bertsekas1996neuro,Munos08-finiteTimeAVI}.
Under mild assumptions such as the sequence of errors having martingale difference, the summation of errors asymptotically cancels out.
This brings a great advantage to deep RL, where properly addressing errors is paramount \citep{fu19a,fujimoto2018addressing}.
However, there is a trade-off between learning speed and robustness in play since the policies change less between iterations with KL regularization.
By setting the KL regularization coefficient as a constant, we lose the ability to dynamically trade-off speed and robustness, and hence the resultant algorithms might not be suitable for robustness-critical problems.
In practice, wild performance oscillation can indeed be observed \citep{nachum2017trust}, since summation is still sensitive to outliers and the errors at the early stage of learning are typically large.

In this paper, we propose dynamically adjusting the KL regularization coefficient according to the error made at each iteration, with the motivation being that for iterations with large error, large KL regularization weight should be imposed to prevent the agent from going in the wrong update direction.
We prove the resulting error propagation bound has the form of \emph{norm-over-weighted-summation} (Theorem~\ref{thm:GVI}), which has the potential to more effectively improve the trade-off between learning speed and robustness than uniform averaging. 

The rest of the paper is organized as follows. We introduce the notations used and review existing constant KL coefficient RL algorithms in Section \ref{sec:notations}.
In Section \ref{sec:policy_iteration}, we study ADP with a dynamic KL coefficient. 
Specifically, we discuss our novel design of the KL coefficient, which is based on the maximum iteration-wise error for weighting the effect of regularization.
Based on the KL coefficient design, in Section \ref{sec:deeprl} we present a practical RL algorithm, Geometric Value Iteration (GVI).
We evaluate GVI on simple mazes and a set of classic control tasks in Section \ref{sec:experiments}.
Our experiments show that GVI can converge faster and more stably and that, moreover, GVI with a deep neural network demonstrates significantly stabilized learning compared to constant regularization, even without target networks. 
Related works and a discussion are given in Section~\ref{sec:related_work}.
Section~\ref{sec:conclusion} presents our conclusions.

\section{Background and Notations}
\label{sec:notations}

We consider a discounted Markov Decision Process (MDP) defined by a tuple $\{\s,\A,P,r,d_0,\gamma\}$, where $\s$ is the finite state space, $\A$ is the finite set of actions, $P\in\Delta_\s^{\s\times\A}$ is the transition kernel (writing $\Delta_X$ as the probability simplex over the set $X$, and $X^Y$ is the set of applications from $X$ to $Y$), $r\in\mathbb{R}^{\s\times\A}$ is the reward function bounded by $r_\text{max}$, $d_0\in\Delta_S$ is the distribution of the initial state, and $\gamma\in(0,1)$ is the discount factor. 
A policy $\pi\in\Delta_\A^\s$ maps states to a distribution over actions, and we write the expectation over trajectories induced by $\pi$ and $d_0$ as $\E_\pi$, where we omit the notation $d_0$ for simplicity.
For a policy $\pi$, the state-action value function is defined as $q_\pi(s,a) = \E_\pi\left[\sum_{t=0}^\infty \gamma^t r(S_t,A_t)\middle \vert S_0=s, A_0=a\right]$, and the (unnormalized) discounted visitation frequency is defined as $d_{\pi}(s) = \E_\pi\left[\sum_{t=0}^{\infty} \gamma^{t} P\left(S_t=s\right)\right]$. 

Following \citet{vieillard2020leverage}, we define a component-wise dot product $ \langle f_1,f_2 \rangle = \\ (\sum_{a} f_1(s,a) f_2(s,a))_{s}\in\R^\s $ for $f_1,f_2\in\R^{\s\times\A}$, which is useful for expectation calculations. 
We define $P v = \left(\sum_{s'} P(s'|s,a) v(s')\right)_{s,a}\in\R^{\s\times\A}$ for $v\in\R^\s$. 
We also define a policy-induced transition kernel $P_\pi$ as $P_\pi q = P\langle \pi, q\rangle$. 
We write the Bellman evaluation operator $T_\pi q = r + \gamma P_\pi q$ and its unique fixed point as $q_\pi$. 
An optimal policy satisfies $\pi_* \in\argmax_{\pi\in\Delta^\s_\A} q_\pi$, and $q_*=q_{\pi_*}$.
We denote the set of greedy policies w.r.t. $q\in\R^{\s\times\A}$ as $\gc(q) = \argmax_{\pi\in\Delta_\A^\s} \langle q,\pi\rangle$. 
When scalar functions are applied to vectors, their applications should be understood in a point-wise fashion.
KL divergence and Shannon entropy are the two most widely used entropy terms for regularization. 
We express KL divergence as $\kl(\pi_1||\pi_2) = \langle \pi_1, \ln \pi_1 - \ln \pi_2 \rangle \in\R^\s$ and Shannon entropy (or simply entropy) as $\mathcal{H}(\pi) = \langle-\pi, \ln\pi\rangle \in \R^\s$.

\paragraph{Mirror Descent Value Iteration}

\citet{vieillard2020leverage} provides a generalized framework for KL-regularized ADP schemes.
The framework, termed Mirror Descent Policy Iteration (MD-PI), is given in Eq.~\eqref{eq:GVI_noentropy1}, where the equation sign indicates the \emph{component-wise update} of a vector.
While MD-PI can also consider the popular Shannon entropy regularization~\citep{haarnoja2017reinforcement}, Shannon entropy does not provide an advantage in the theoretical error propagation analysis of MD-PI~\citep{vieillard2020leverage}.
Accordingly, in this paper, we focus only on the KL regularization.

\begin{equation}
    \text{ MD-PI } \;
    \begin{cases}
        \pi_{k+1} &= \gc_{\pi_k}^{\lambda}(q_k) \\
        q_{k+1} &= (T^{\lambda}_{\pi_{k+1}\sep\pi_k})^m q_k + \epsilon_{k+1}
    \end{cases}.\label{eq:GVI_noentropy1}
\end{equation}
In Eq.~\eqref{eq:GVI_noentropy1}, we start from a uniform policy $\pi_0$ and evaluate the next policy by applying the greedy operator $\gc_\mu^{\lambda}(q) = \argmax_{\pi\in\Delta^\s_\A} \left(\langle \pi, q\rangle - \lambda \kl(\pi||\mu)\right)$ with an arbitrary baseline policy $\mu \in\Delta^\s_\A$.
Usually, $\mu$ is chosen as the previous policy.
With the obtained regularized greedy policy, we evaluate its action value function $q_{k+1}$ by applying $m$ times the regularized Bellman operator $T_{\pi\sep\mu}^{\lambda_k} q = r + \gamma P \left(\langle \pi, q\rangle - \lambda_k \kl(\pi||\mu)\right)$. 
Setting $m=1, \infty$ corresponds to value iteration and policy iteration schemes, respectively.
Setting $m$ to any other value implies the use of approximate modified policy iteration \citep{puterman1978modified,scherrer2015approximate}.
We call Eq.~\eqref{eq:GVI_noentropy1} with $m=1$ the Mirror Descent Value Iteration (MD-VI).
The error term $\epsilon_{k+1}$ is a vector of the same shape as the action value function,
and it is typically assumed that the greedy step is free of error \citep{vieillard2020leverage}.

By the Fenchel conjugacy \citep{cvx-opt-boyd}, the greedy policy $\pi_{k+1}$ can be analytically obtained as $\gc_\mu^{\lambda}(q) \propto \mu\exp{(\frac{q}{\lambda})}$~\citep{geist19-regularized}.
By choosing $\mu = \pi_{k}$, a direction induction shows that the MD-PI policy $\pi_{k+1}$ averages all previous $Q$-values as $\pi_{k+1} \propto \pi_k \exp \frac{q_k}{\lambda} \propto \dots \propto \exp\frac{1}{\lambda}\sum_{j=0}^k q_j$.
Since the errors are additive, $\pi_{k+1}$ also averages errors from previous iterations.
Indeed, the following theorem formally shows that the finite-time bound of MD-VI depends on the {\it norm of the average} of the accumulated errors (the extension to $m>1$ remaining an open question).
\begin{theorem}[\citet{vieillard2020leverage}]
    \label{thm:mdmpi_noEntropy}
    Define the maximum value of $\|q_k\|_\infty$ as $q_\text{max}$. 
    The $\ell_\infty$-bound of MD-VI is 
    \begin{equation}\label{eq:constant_bound}
    \| q_* - q_{\pi_{k+1}}\|_\infty \leq \frac{2}{(1-\gamma)}{\frac{1}{k}}\left(\left\|\sum_{j=1}^k \epsilon_j\right\|_\infty + 2 q_\text{max}+\lambda\gamma\ln |\mathcal{A}| \right).
    \end{equation}
\end{theorem}
In Theorem~\ref{thm:mdmpi_noEntropy}, the optimality gap $\| q_* - q_{\pi_{k+1}}\|_\infty$ is expressed in terms of errors $\frac{1}{k}\sum_{j=1}^k \epsilon_j$, which are averaged with respect to the uniform distribution.
This corresponds to having a constant coefficient-KL regularization throughout learning (i.e., fixing $\lambda$).
Under mild assumptions, such as the sequence of errors having martingale difference under the natural filtration \citep{azar2012dynamic}, the summation of errors asymptotically cancels out.
However, the asymptotic cancelation of errors happens only under specific conditions.
When the conditions are not satisfied, having a constant coefficient assumes the errors contribute equally (i.e., $\frac{1}{k}$) to the gap $\| q_* - q_{\pi_{k+1}}\|_\infty$, which is often not the case, since in the early stages of learning the errors are typically large and require more attention.

Our motivation comes from the intuition of weighting down large errors using large regularization coefficients $\lambda_{k}$, and thus the weighted average of errors $\frac{1}{\sum_{j=1}^{k}{1}/{\lambda_j}}\|\sum_{j=1}^{k}{\epsilon_j}/{\lambda_{j}}\|_\infty$ could be much smaller than that of uniform averaging $\frac{1}{k}\|\sum_{j=1}^{k}{\epsilon_j}\|_\infty$.
This corresponds to setting a different KL coefficient for each iteration.
Intuitively, different coefficients allow for more robust convergence and potentially faster convergence since the magnitude of $\frac{1}{\sum_{j=1}^{k}{1}/{\lambda_j}}({\epsilon_j}/{\lambda_{j}})$ could be much smaller than $\frac{1}{k}\epsilon_j$ if we are allowed to specify the coefficient $\lambda_{j}$. 
This motivation prompts the use of a {\it dynamic error-aware} KL coefficient design that is detailed in Section \ref{sec:policy_iteration}.

\section{Dynamic Error-Aware KL Regularization}\label{sec:policy_iteration}

While MD-VI is generally robust against zero-mean errors because the summation of errors asymptotically cancels out, in some situations the errors fail to cancel each other out and result in bad performance of MD-VI.
As a concrete example, consider the following errors induced every $K$ step: 
\begin{equation}\label{eq:noise}
\begin{cases}
    \epsilon_{k} \sim \text{unif}(0, K) \; &\text{if}\; k=K, 2K, \dots \\
    \epsilon_{k} = 0 \; &\text{otherwise}
\end{cases}.
\end{equation}
This artificial example can be likened to a two-state MDP case (Figure~\ref{fig:maze_mdp}) where the agent starting from state 1 continues to loop onto itself with zero cost and probability $\frac{K-1}{K}$, and with probability $\frac{1}{K}$ the agent goes to state 2 with cost $\texttt{unif}(0,K)$ and then back again to state 1.
% Eq. (\ref{eq:noise}) may then be used to model practical cases such as where errors associated with network weights updates increase after a regular period of time due to the changes of tasks.

\begin{figure}[t]
    \centering
    \begin{minipage}{0.48\textwidth}
        \centering
        \includegraphics[width=0.85\textwidth]{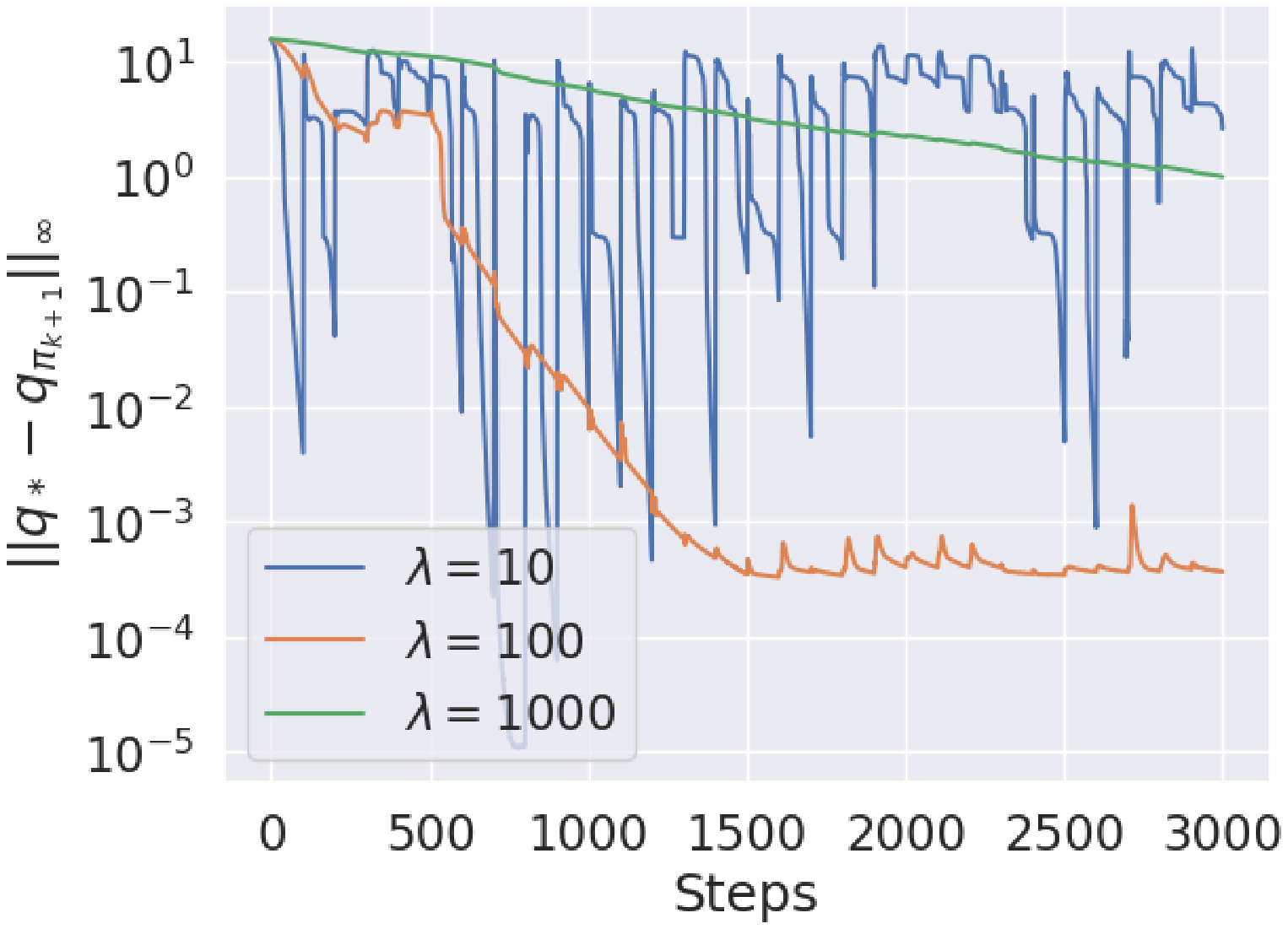}
        \caption{MD-VI in maze under different values of $\lambda$ with error generated by Eq.~\eqref{eq:noise}.
        In this case, the optimal regularization strategy is time-dependent.}
        \label{fig:md-mpi-maze}
    \end{minipage}\hfill
    \begin{minipage}{0.48\textwidth}
        \centering
        \includegraphics[width=0.95\textwidth]{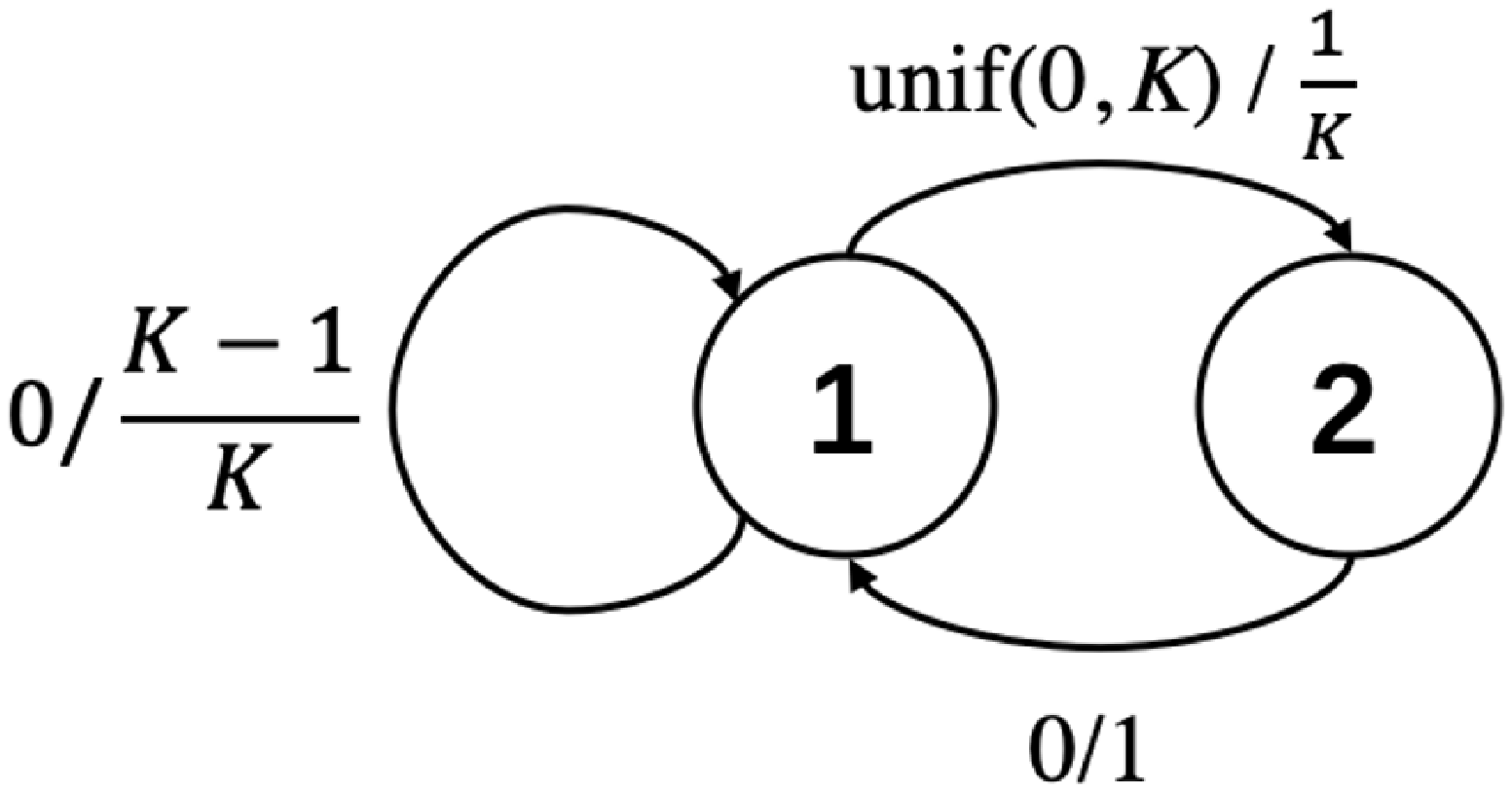}
        \caption{A two-state MDP instantiation of Eq.~\eqref{eq:noise}. Numbers on the transitions indicate cost/probability.
        Such errors might occur in updating weights of neural networks.}
        \label{fig:maze_mdp}
    \end{minipage}
\end{figure}
Figure~\ref{fig:md-mpi-maze} illustrates the optimality gap of MD-VI under randomly generated $5\times5$ mazes and errors of Eq.~\eqref{eq:noise} with $K=100$ (see the environment's details in Appendix~\ref{apdx:maze}). 
In this simple setup, trials with small regularization coefficients (blue line) fail to converge due to performance oscillation brought by the error in Eq.~\eqref{eq:noise}, while larger regularization (yellow and green lines) achieves convergence to the optimal policy but at a much slower rate.
A suitable strategy in this example is obviously an {\it error-aware} regularization strategy: being conservative only when the errors are present, and greedy otherwise.

To overcome the limitations of the constant-weight regularization scheme, we study the following regularized PI scheme with a dynamic KL coefficient $\lambda_{k}$:
\begin{equation}
    \begin{cases}
        \pi_{k+1} &= \gc_{\pi_k}^{\lambda_k}(q_k)
        \\
        q_{k+1} &= (T^{\lambda_k}_{\pi_{k+1}\sep\pi_k})^m q_k + \epsilon_{k+1}
    \end{cases}.\label{eq:dkl-mdmpi}
\end{equation}
We derive the error-aware regularization bound for the policy iteration case.
The following theorem provides a bound on the optimality gap of Eq.~\eqref{eq:dkl-mdmpi}.

\begin{theorem}\label{thm:GVI}
    Define $\eta_k=1/\lambda_k$ and $Z_k=\sum^k_{j=0}\eta_j$. 
    The $\ell_\infty$-bound of Eq.~\eqref{eq:dkl-mdmpi} with $m=1$ is 
    \begin{equation}\label{eq:GVI-bound}
    \| q_* - q_{\pi_{k+1}}\|_\infty \leq \frac{2}{(1-\gamma)}\frac{1}{Z_{k}}\left(\left\|\sum_{j=1}^k \eta_j\epsilon_j\right\|_\infty + (\eta_{k+1} + \eta_0 + \sum^k_{j=0}\left|\eta_{j+1}-\eta_j\right|)q_\text{max}+\gamma\ln|A|\right).
    \end{equation}
\end{theorem}
\begin{proof}
    See Appendix~\ref{apdx:gpi-proof} for the proof.
\end{proof}
Note that this bound generalizes the bound of MD-VI, since Eq.~\eqref{eq:GVI-bound} matches Eq.~\eqref{eq:constant_bound} when $\eta = 1 / \lambda$.
As with MD-VI, this bound features a linear dependency of errors on the horizon $\frac{1}{1-\gamma}$. 
The main difference is the {\it weighted average} of the errors instead of the {\it uniform average} for MD-VI.
This weighted average error term intuitively motivates the design of regularization coefficients $\eta_k=1/\lambda_k$.

First, minimizing the optimality gap implies minimizing $\|\frac{1}{Z_k}\sum_{j=1}^k {\eta_j}\epsilon_j\|_\infty$, which is the norm of the weighted arithmetic mean of errors.
Because $\epsilon_j$ is a random variable for all $j$, with ${\eta'}_j=\frac{\eta_j}{Z_k}$, the mean and the variance of the weighted arithmetic mean are given by $\sum^k_{j=1}{\eta'}_j\E[\epsilon_j]$ and $\sum^k_{j=1}{{\eta'}_j}^2 \var[\epsilon_j]$, respectively.
This in turn suggests that $\eta_j$ should be inversely scaled according to the magnitude of $\epsilon_j$ to restrict potentially erroneous updates where errors have huge means or variances.
By recalling $\epsilon_j$ is a vector, we scale $\eta_j$ according to the infinity norm as $\eta_{j} = \frac{1}{\alpha_1\|\epsilon_{j}\|_{\infty}}$, where $\alpha_1$ is used for uniformly scaling all of the coefficients.
Note that $\alpha_1$ does not appear in the error-dependent term since it appears in both numerator $\eta_{j}$ and denominator $Z_{k}$.

In deep RL, hyperparameters are typically and gradually decayed instead of changed abruptly.
This highlights the importance of stability in learning with neural networks, which we address here.
Given the above design choice, we impose an additional constraint that no huge increase from $\eta_k$ to $\eta_{k+1}$ is allowed: such an increase during learning can be measured by $\frac{1}{Z_k}\sum^{k}_{j=0}\left|\eta_{j+1}-\eta_j\right|$, which appears in the second term of the error bound Eq.~\eqref{eq:GVI-bound}.
We do not allow the term to diverge by restricting $\eta_{j+1} > 2\eta_j$, which makes $\frac{1}{Z_k}\sum^{k-1}_{j=0}\left|\eta_{j+1}-\eta_j\right|$ larger than $1$.
To this end, we gradually decay the regularization coefficient by introducing another hyperparameter $\alpha_2$, such that $\lambda_{k} = \alpha_2\lambda_{k-1}$ with $\alpha_2 \in (0, 1)$ generally close to one.

The above-mentioned design choices can be summarized as the following dynamic KL coefficient design:
\begin{equation}\label{eq:kl-dynamic}
    \lambda_k = \max(\alpha_1\|\epsilon_k\|_\infty, \alpha_2\lambda_{k-1}),
\end{equation}
where $\alpha_1 \in \R^+$ and $\alpha_2 \in (0, 1)$.
% The values used for experiments are provided in the Appendix.
%In Eq.~\eqref{eq:kl-dynamic}, $\max(\|\epsilon_k\|_\infty, \cdot)$ and $\max(\cdot, \alpha_2\lambda_{k-1})$ reflects the first and the second guideline, respectively.

\section{Geometric Value Iteration}\label{sec:deeprl}

In this section we propose a novel algorithm based on the dynamic KL regularization coefficient design of the previous section.
While it is straightforward to incorporate it in the general MD-VI scheme of Eq.~\eqref{eq:dkl-mdmpi}, a crucial subtlety stands in the way of achieving better performance: we know $\pi_{k+1} \propto \exp(\sum_{j=1}^{k}q_{j})$ from Section \ref{sec:notations}, which requires remembering all previous value functions. 
In practice, approximation such as information projection would have to be used \citep{vieillard2020momentum}, which brings errors to the policy update step.

Leveraging the very recent idea of \emph{implicit KL regularization} \citep{vieillard2020munchausen}, it is possible to circumvent the need for remembering all previous values in MD-VI by augmenting the reward with a log-policy term, whose formulation is given in Eq.~\eqref{eq:MPI}.
The reward function is augmented by the term $\ln\pi_{k+1}$ weighted by the KL coefficient $\lambda$:
\begin{equation}\label{eq:MPI}
\;
\left\{\begin{array}{l}
\pi_{k+1}=\argmax_{\pi\in\Delta^\s_\A}\left\langle\pi, q_{k}\right\rangle - \lambda \hc(\pi) \\
q_{k+1}= {\lambda\ln \pi_{k+1}} + r + \gamma P\left\langle\pi_{k+1}, q_{k}-\lambda \ln \pi_{k+1} \right\rangle
\end{array}\right.
.
\end{equation}
Eq.~\eqref{eq:MPI} corresponds to implicitly performing KL regularization, and hence there is no need for remembering previous values, that is, computing the term $\ln\pi_{k+1}$ suffices.

While Eq.~\eqref{eq:MPI} provides an easy-to-use scheme for our dynamic KL coefficient by replacing $\lambda$ with $\lambda_k$, the term $\lambda_{k}\ln\pi_{k}$ could cause numerical issues when $\lambda_{k}$ has a huge value.
For numerical stability, we propose further transforming Eq.~\eqref{eq:MPI} as follows:

\begin{equation}\label{eq:GVI2}
\left\{\begin{array}{l}
\pi_{k+1}=\argmax_{\pi\in\Delta^\s_\A}\left\langle\pi, q_{k}\right\rangle + \hc(\pi)\\
q_{k+1}=  \ln \pi_{k+1} + \frac{r}{\lambda_{k+1}} + \frac{\lambda_k}{\lambda_{k+1}} \gamma P\left\langle\pi_{k+1}, q_{k}-\ln \pi_{k+1} \right\rangle
\end{array}\right.
.
\end{equation}
Additional clipping might also be necessary to restrict the magnitude of $\ln\pi_{k+1}$.
We can show that the scheme of Eq.~\eqref{eq:GVI2} is equivalent to the formulation of Eq.~\eqref{eq:dkl-mdmpi}, which we formally state below.

\begin{theorem}\label{thm:GVI2}
For any $k\geq 0$, by defining ${q'}_k=\lambda_{k+1}\left(q_k-\ln \pi_k\right)$, we have
\begin{equation}
\eqref{eq:GVI2} \Leftrightarrow 
\left\{\begin{array}{l}
\pi_{k+1}=\argmax_{\pi\in\Delta^\s_\A}\left\langle\pi, {q'}_{k}\right\rangle - \lambda_k\kl({\pi}\|{\pi_k})\\
{q'}_{k+1}=r + \gamma P\left\langle\pi_{k+1}, {q'}_{k} - \lambda_{k} \kl({\pi_{k+1}}\|{\pi_{k}})\right\rangle 
\end{array}\right. .
\end{equation}
\end{theorem}
\begin{proof}
    See Appendix~\ref{apdx:GVI2} for the proof.
\end{proof}

By dynamically adjusting the KL coefficient, Eq.~\eqref{eq:GVI2} mitigates issues brought by various sources of error and improves learning stability.
One more problem remains for making Eq.~\eqref{eq:GVI2} practically applicable. 
In tuning the KL coefficient Eq.~\eqref{eq:kl-dynamic}, the magnitude information of $\|\epsilon_{k+1}\|_\infty$ is typically unavailable.
Taking inspiration from a very recent work \citep{vieillard2020deep}, we approximately compute this error by moving average TD-error from batches.
Hence, we approximate $\|\epsilon_{k+1}\|_\infty$ by the maximum absolute TD error $\|\epsilon_{\text{TD}, i}\|_\infty$, where $i$ indicates the $i$th batch.
In summary, our Geometric Value Iteration (GVI) iterates as follows:
\begin{equation}\label{eq:GVI}
\text{GVI}\;
\left\{\begin{array}{l}
\lambda_{k+1} = \max(\alpha_1\|\epsilon_\text{TD, k}\|_\infty, \alpha_2\lambda_{k}) \\
\pi_{k+1}=\argmax_{\pi\in\Delta^\s_\A}\left\langle\pi, q_{k}\right\rangle + \hc(\pi)\\
q_{k+1}=\ln \pi_{k+1} + \frac{r}{\lambda_{k+1}} + \frac{\lambda_k}{\lambda_{k+1}} \gamma P\left\langle\pi_{k+1}, q_{k}-\ln \pi_{k+1} \right\rangle
\end{array}\right. .
\end{equation}
The name {\it Geometric} comes from the fact that GVI mixes two policies by weighted {\it geometric} mean as $\pi_{k+1} = \gc_{\pi_k}^{\lambda_k}(q_k) \propto (\pi_k)^{1-\zeta_k}\left(\gc_{\pi_k}^{\lambda}(q_k)\right)^{\zeta_k}$, where $\lambda / \zeta_k = \lambda_k$.
% Focusing only on the policy update step, GVI can be seen as a geometric mean version of Conservative Policy Iteration~\citep{kakade-cpi}, which mixes two policies by weighted arithmetic mean.

We now present the implementation of Eq.~\eqref{eq:GVI} using deep networks, or Deep GVI (DGVI). 
Suppose $Q$-values are estimated by an online $Q$ network parameterized by weight vector $\theta$ and the transition data are stored in a FIFO replay buffer $\mathcal{B}$.
DGVI minimizes the following loss function:
\begin{align}\label{eq:qnet}
    L_{\theta} &= \E_{(s, a, r, s')\sim\mathcal{B}}\left[\left(q_{\theta}(s,a)-y(s, a)\right)^2\right],\\
    \text{ where }\; y(s, a)&=\ln \pi (a|s) + \frac{r(s,a)}{\lambda'} + \frac{\lambda}{\lambda'}\gamma\E_{a' \sim \pi(\cdot|s')} \left[q_{\bar{\theta}}(s', a') - \ln{\pi(a'|s')}\right],
\end{align}
where $\bar{\theta}$ indicates the weight vector of the target network and $\pi\propto \exp(q_{\bar{\theta}})$ is the greedy policy.
$\lambda$ and $\lambda'$ are the previous and current KL coefficients, respectively.
The use of slowly updated target network $q_{\bar{\theta}}$ in the target $y(s, a)$ is conventional for stability purposes. 
The parameters $\bar{\theta}$ are either infrequently copied from $\theta$ or obtained by Polyak averaging $\bar{\theta}$.
While target networks could be used to further enhance the performance, in our experiments we explicitly remove target networks to highlight the error-robustness of DGVI.

Taking inspiration from \citep{vieillard2020deep}, we use the moving average of maximum batch TD errors to approximate the maximum error based on Eq.~\eqref{eq:GVI2}:

\begin{equation}\label{eq:prac-lambda}
\begin{aligned}
   \lambda' &\leftarrow (1 - \nu) \lambda' + \nu \max(\alpha_1 \|\epsilon_\text{TD}\|_\infty, \alpha_2\lambda)\\
   \lambda &\leftarrow (1 - \nu_\text{slow}) \lambda + \nu_\text{slow} \lambda'
\end{aligned},
\end{equation}
where $\|\epsilon_\text{TD}\|_\infty$ is the maximum absolute TD error in a batch, and $\nu$ and $\nu_\text{slow}$ are learning rates for $\lambda$ and $\lambda'$, respectively.
We summarize the algorithm of DGVI in Algorithm.~\ref{alg:dGVI}.

\begin{algorithm}[tb]
\begin{algorithmic}[1]
\caption{Deep Geometric Value Iteration} \label{alg:dGVI}
% \REQUIRE 
\STATE Initialize $\theta$, $\lambda$ and $\lambda'$
\FOR{each iteration}
    \STATE Collect transitions and add them to $\mathcal{B}$
	\FOR{each gradient step}
	    \STATE Compute the maximum absolute TD error $\|\epsilon_\text{TD}\|_\infty$ in a minibatch.
	    \STATE Update $\lambda$ and $\lambda'$ using Eq.~\eqref{eq:prac-lambda}.
    	\STATE Update $\theta$ with one step of SGD using Eq.~\eqref{eq:qnet}
	\ENDFOR
\ENDFOR
\end{algorithmic}
\end{algorithm}

\section{Experiments}\label{sec:experiments}

This section empirically studies the proposed GVI with tabular and deep implementation. 
We wanted to evaluate the effectiveness of our error-aware KL coefficient design in handling the trade-off between learning speed and stability.
For didactic purposes, we first evaluated GVI on a tabular maze environment that is the same as the one used in Figure~\ref{fig:md-mpi-maze}.
The tabular experiments serve to verify that GVI can better handle the trade-off problem between learning speed and robustness than the constant KL coefficient scheme.
We then conducted an experiment on classic control tasks from OpenAI Gym benchmarks~\citep{brockman2016openai} to observe the behavior of GVI with deep implementation.
For the deep RL experimentation, we consider GVI as a variation of Munchausen-DQN (M-DQN)~\citep{vieillard2020munchausen} and thus take M-DQN as our baseline.

\paragraph{Tabular Experiments}
Figure~\ref{fig:maze_comparison} investigates the optimality gap of GVI with varying conditions.
For GVI, we also included the investigation of the introduced hyperparameters $\alpha_1$ and $\alpha_2$ and their impact on performance.
Although they do not play any role in error analysis, in practice they can have a large effect on the trade-off between speed and stability.

The left graph in Figure~\ref{fig:maze_comparison} compares the best behavior of GVI with MD-VI, where the parameters of GVI are fine-tuned to yield the empirically best performance.
The figure shows that GVI achieves faster and more robust convergence than MD-VI under a certain hyperparameter. 
GVI reaches the minimum optimality gap in around $50$ steps and keeps the value under $10^{-3}$.
On the other hand, MD-VI suffers from the trade-off between speed and stability.
While $\lambda=50$ reaches the minimum optimality gap close to that of GVI, it reaches it in around $1000$ steps and is thus much slower than GVI.
MD-VI with $\lambda=30$ converges faster, but the optimality gap oscillates and exceeds $10^{-2}$.
Therefore, it can be safely concluded that the constant KL coefficient scheme MD-VI cannot outperform GVI.
 
The middle and the right graphs are plotted to investigate the behavior of GVI with different $\alpha_1$ and $\alpha_2$.
GVI with a small $\alpha_1$ never reaches the optimal value, while experiments with the small $\alpha_2$ obtain a small optimality gap at the cost of huge oscillation. 
These are expected since $\lambda_k$ corresponds to the learning rate of the updates~\citep{kozunoCVI}, and $\alpha_2$ decides how long the conservativeness remains after detecting large errors. 

\begin{figure}[t]
    \centering
    \begin{minipage}[c]{0.32\linewidth}
        \includegraphics[width=\linewidth]{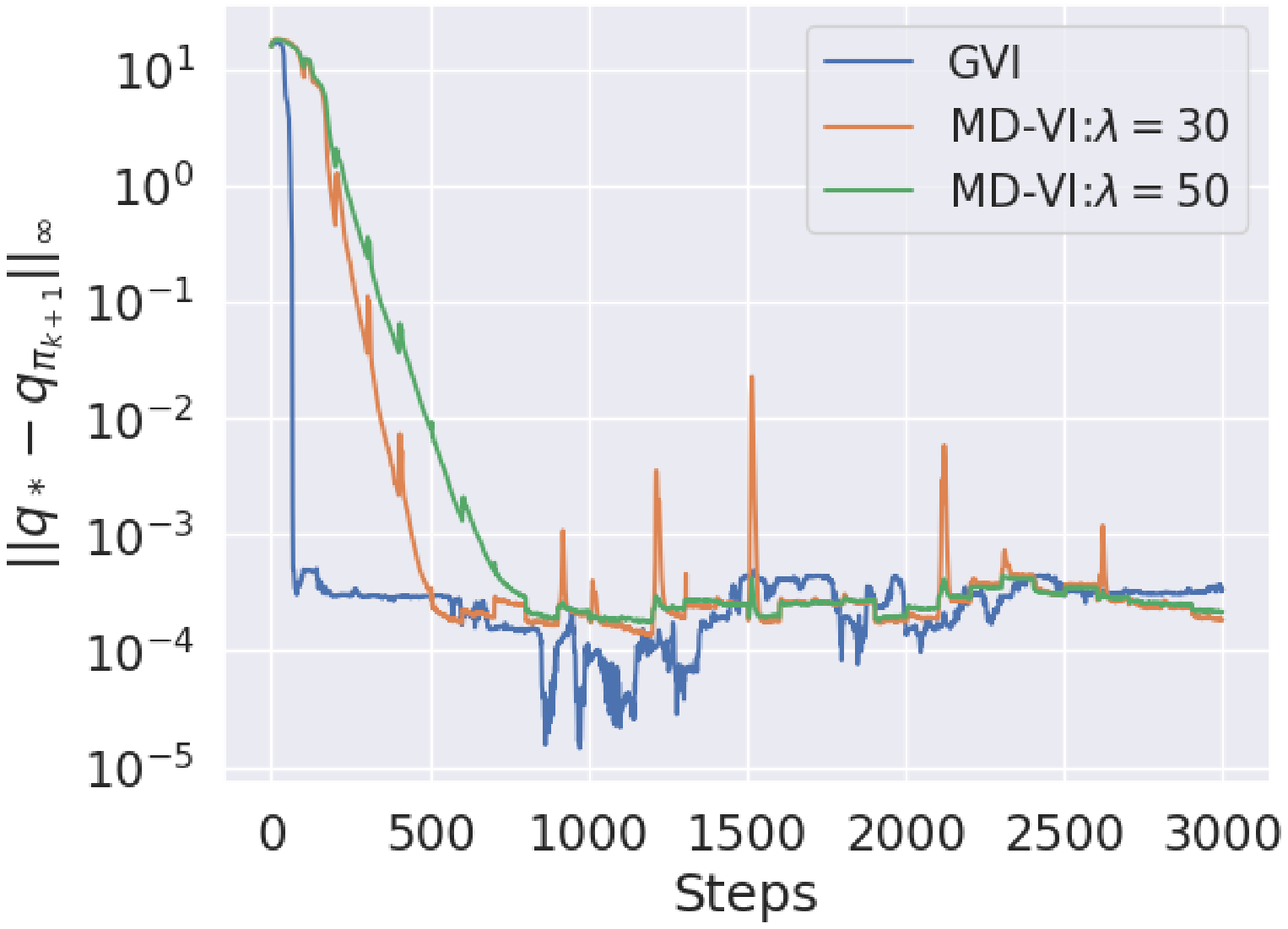}
    \end{minipage}
    \begin{minipage}[c]{0.32\linewidth}
        \includegraphics[width=\linewidth]{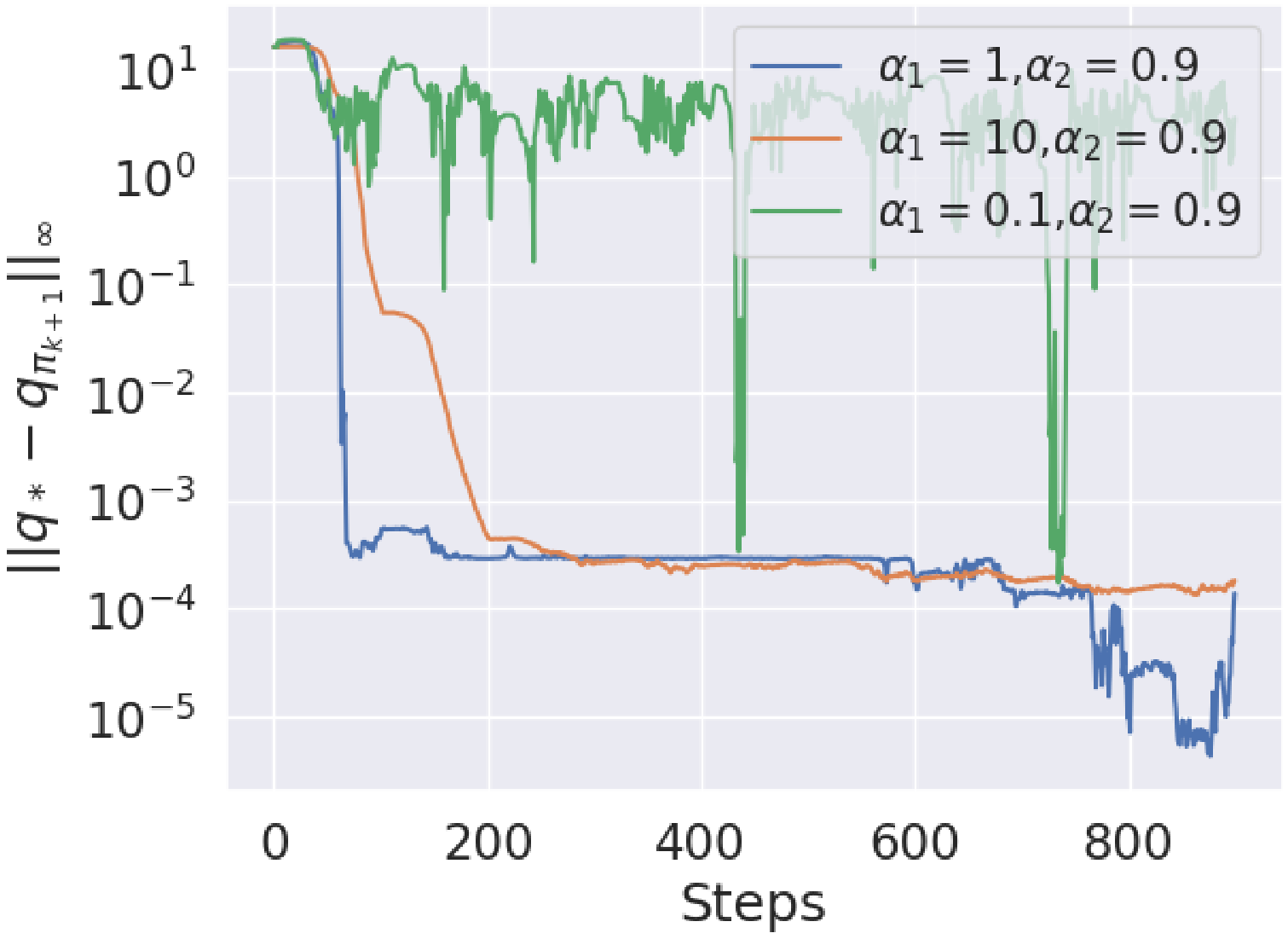}
    \end{minipage}
    \begin{minipage}[c]{0.32\linewidth}
        \includegraphics[width=\linewidth]{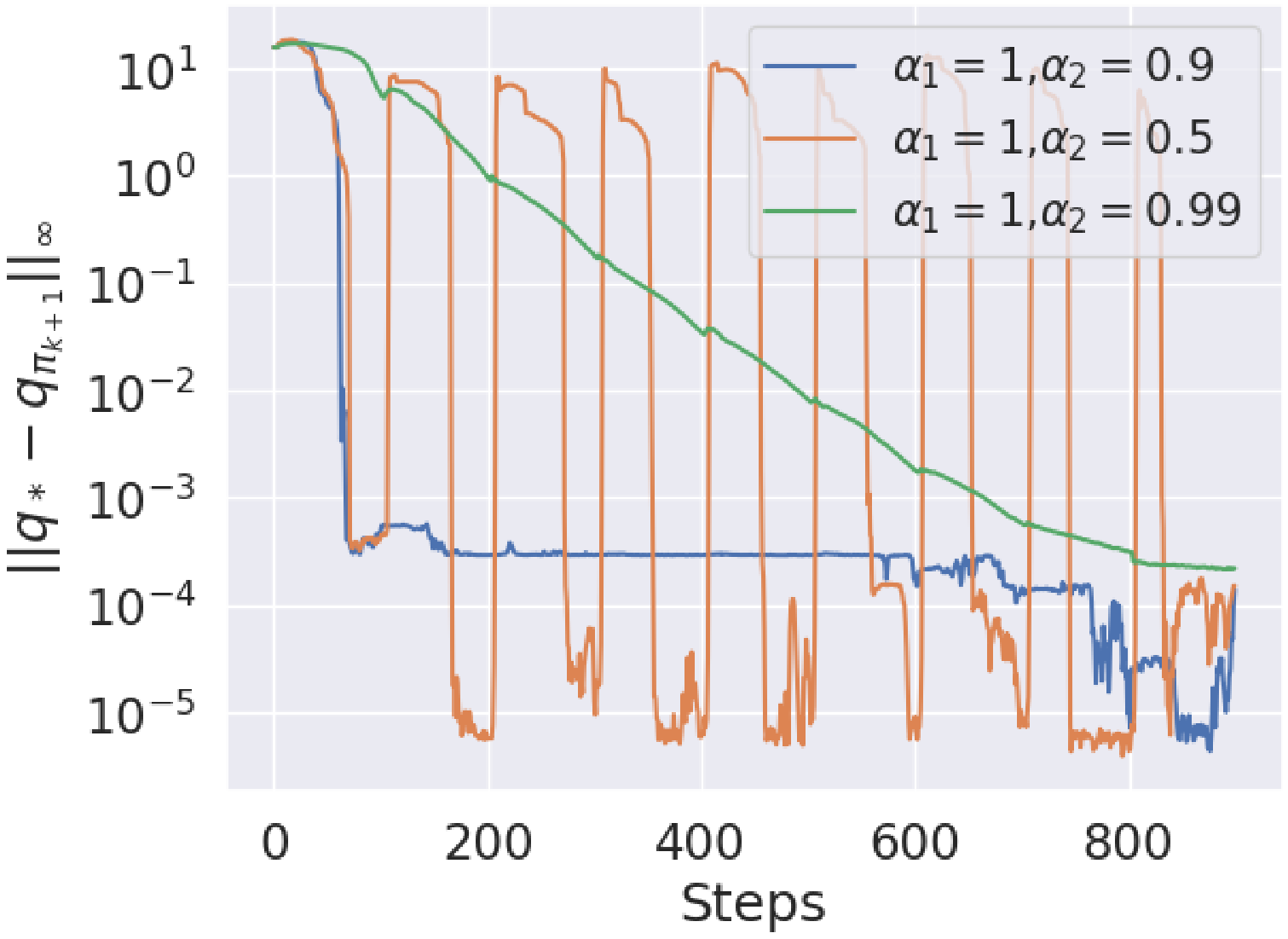}
    \end{minipage}
    \caption{(Left) Performance comparison of GVI ($\alpha_1=2, \alpha_2=0.9$) and MD-PI($\lambda=30$, $\lambda=50$). (Middle) GVI with different $\alpha_1$. (Right) GVI with different $\alpha_2$.}
    \label{fig:maze_comparison}
\end{figure}

\paragraph{Deep RL Experiments}

Using a set of classic control benchmarks~\citep{brockman2016openai}, we examine the DGVI of Algorithm.~\ref{alg:dGVI} against the constant KL coefficient algorithm of M-DQN~\citep{vieillard2020munchausen}.
We choose the LunarLander-v2, CartPole-v1, and Pendulum-v0 environments as our benchmarks.
Since our DGVI and M-DQN support only discrete action space environments, we discretized the continuous action space of Pendulum-v0 into five discrete actions.
For each seed we perform 10 evaluation rollouts every 300 environment steps.
For a fair comparison, all of the algorithms share the same hyperparameters except the KL regularization.
To highlight the robustness of algorithms against estimation errors, we explicitly remove target networks from the algorithms, even though such networks provide a key ingredient to the success of modern deep RL \citep{mnih2015human,pmlr-v80-haarnoja18b}.
All figures are plotted by averaging results from five independent random seeds for statistical results.
We list the set of hyperparameters in Appendix~\ref{apdx:hypers}.

Figure~\ref{fig:return_comparison} shows the learning curves of algorithms and the corresponding KL regularization of DGVI.
Compared to the constant regularized algorithms, GVI achieves more stable learning in DiscretePendulum and CartPole.
Notably, GVI has smaller regularization in DiscretePendulum and CartPole than $\lambda=10$.
This indicates that the dynamic change of the KL coefficient is more important than its magnitude.

To observe how the dynamic KL coefficient improves stability, we evaluated the maximum absolute TD error $\|\epsilon_\text{TD}\|_\infty$ as shown in Figure~\ref{fig:td}.
Compared to constant regularized algorithms, the error of DGVI proves to be much smaller during learning.
This result agrees well with how DGVI updates the network by Eq.~\eqref{eq:qnet}: the bootstrap is scaled by $\frac{\lambda}{\lambda'}$, which becomes small when DGVI encounters large errors.
For a better understanding of the effect on the bootstrap, consider an extreme case where a significantly huge error is induced and $\lambda'$ is infinite.
Then, the loss becomes $L_{\theta} \approx \E_{(s, a)\sim\mathcal{B}}\left[\left(q_{\theta}(s,a)-\ln \pi(s, a)\right)^2\right]=\E_{s\sim\mathcal{B}}\left[\left(\ln \sum_{a\in \A} \exp\left(q_\theta (s, a)\right)\right)^2\right]$, and thus the new $q_\theta$ will have smaller values.
GVI thus tends to underestimate the state and action pairs where huge errors are expected, which is assumed to prevent bad updates from quickly spreading to downstream $Q$-values.
We can conclude that the proposed mechanism renders DGVI stable even without target networks.

\begin{figure}[t]
    \centering
    \begin{minipage}[c]{0.32\linewidth}
        \includegraphics[width=\linewidth]{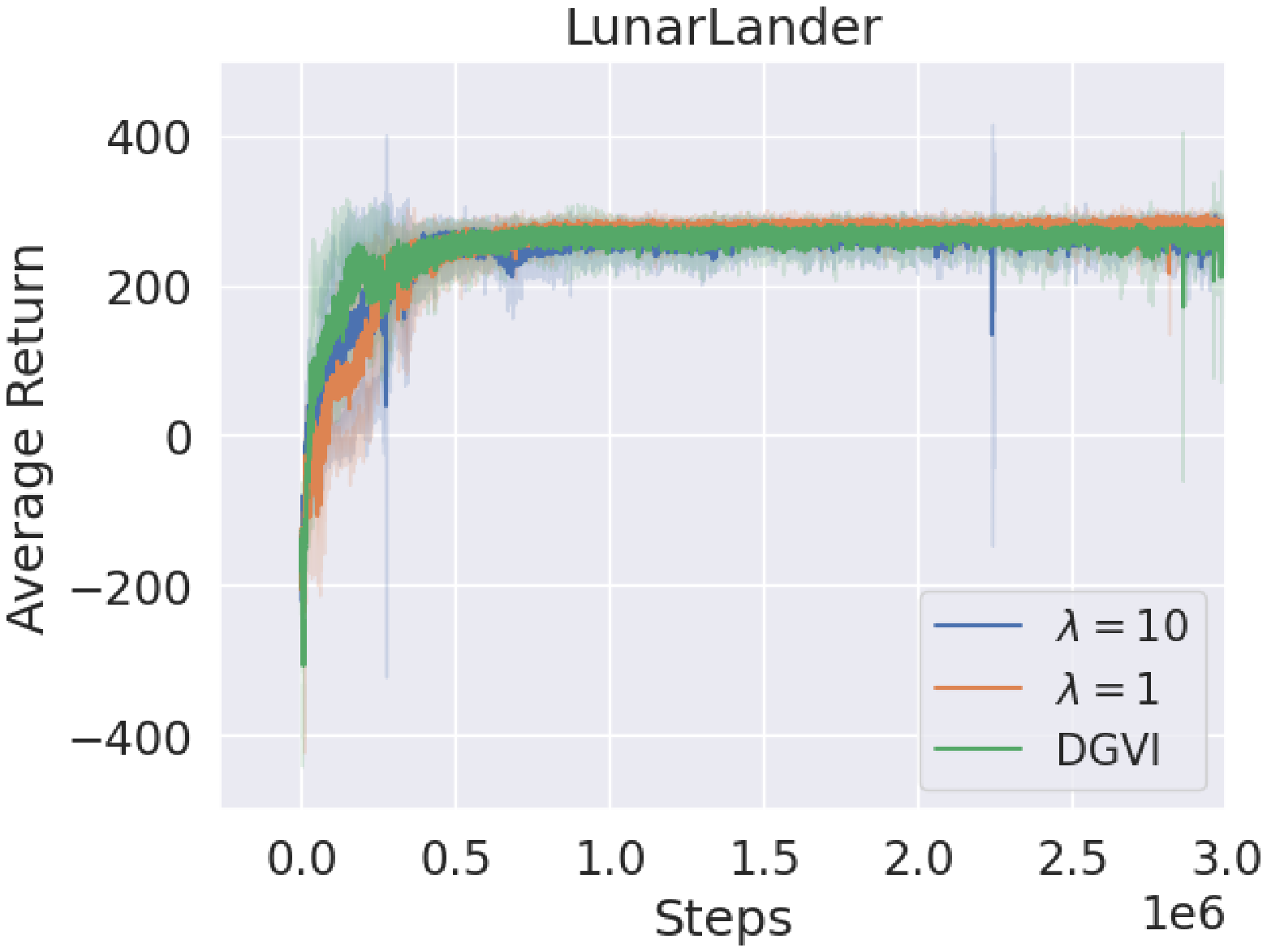}
    \end{minipage}
    \begin{minipage}[c]{0.32\linewidth}
        \includegraphics[width=\linewidth]{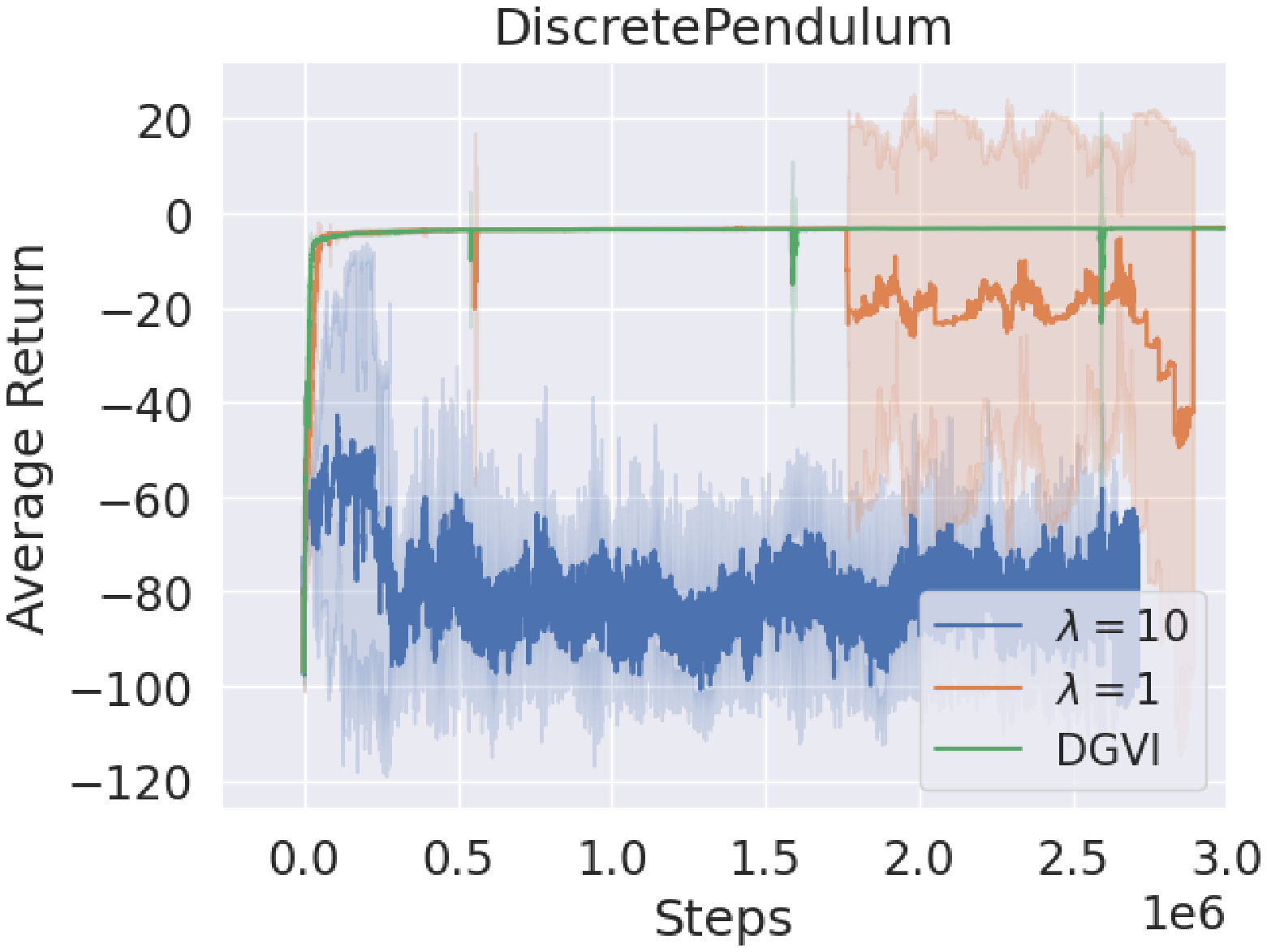}
    \end{minipage}
    \begin{minipage}[c]{0.32\linewidth}
        \includegraphics[width=\linewidth]{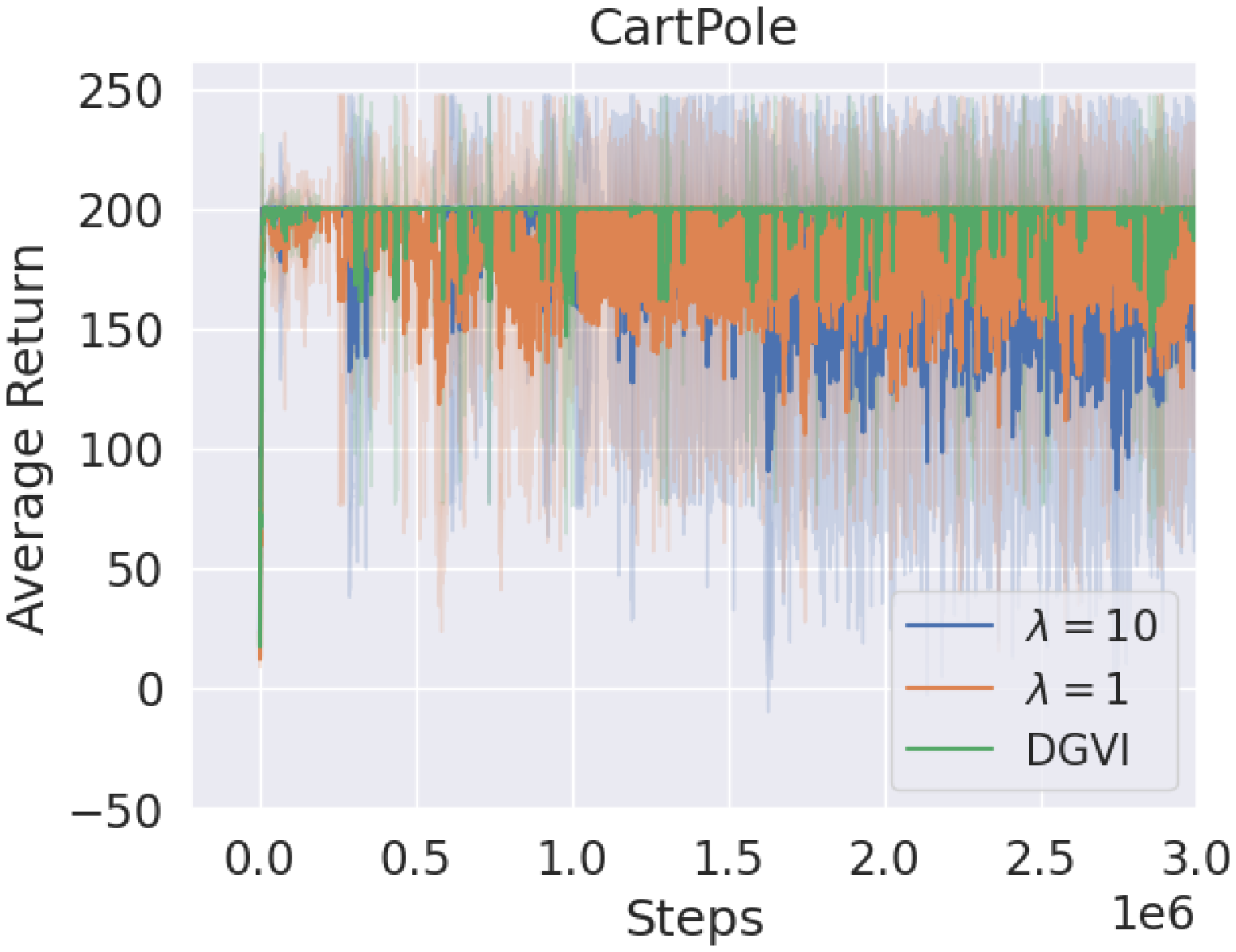}
    \end{minipage}
    \begin{minipage}[c]{0.32\linewidth}
        \includegraphics[width=\linewidth]{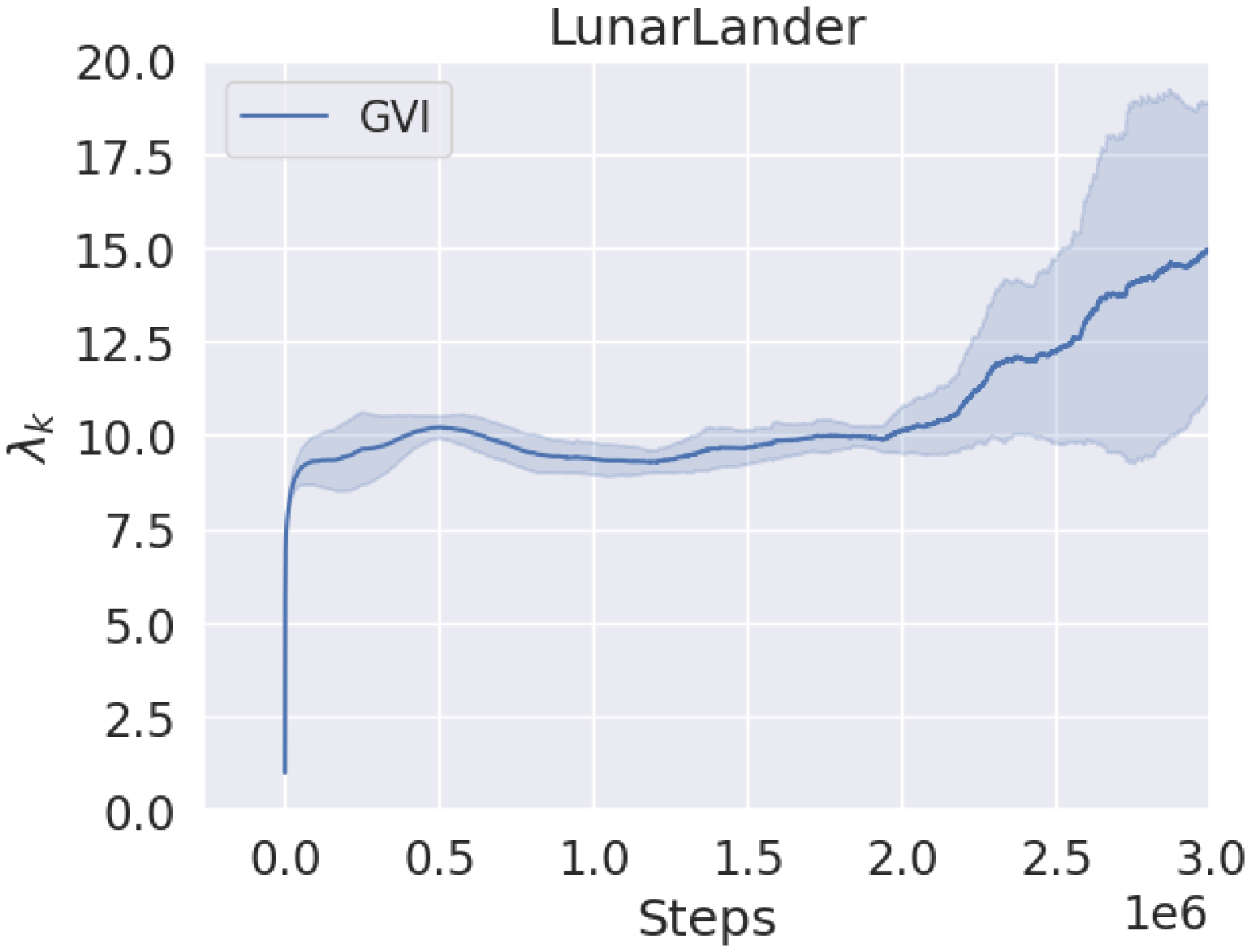}
    \end{minipage}
    \begin{minipage}[c]{0.32\linewidth}
        \includegraphics[width=\linewidth]{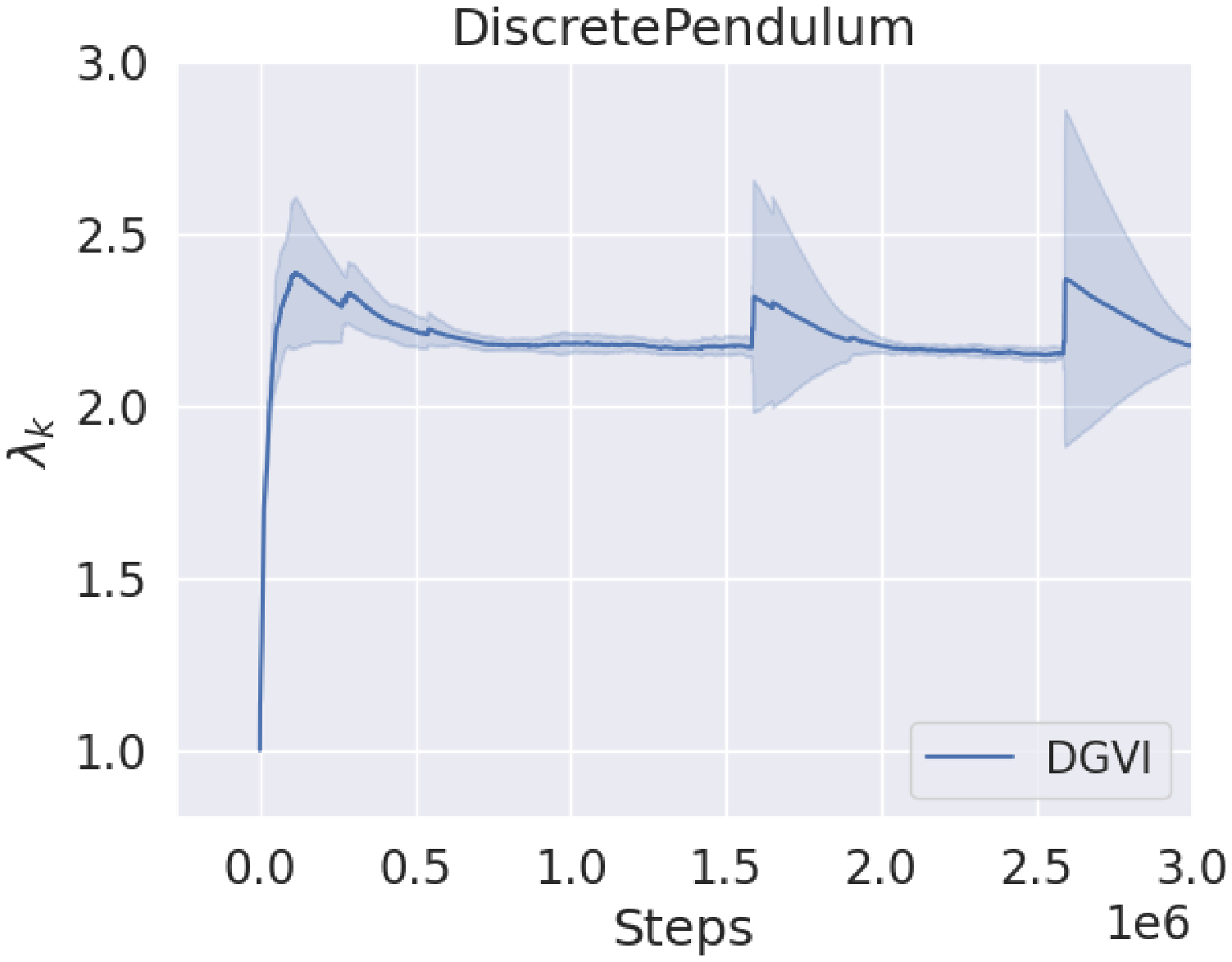}
    \end{minipage}
    \begin{minipage}[c]{0.32\linewidth}
        \includegraphics[width=\linewidth]{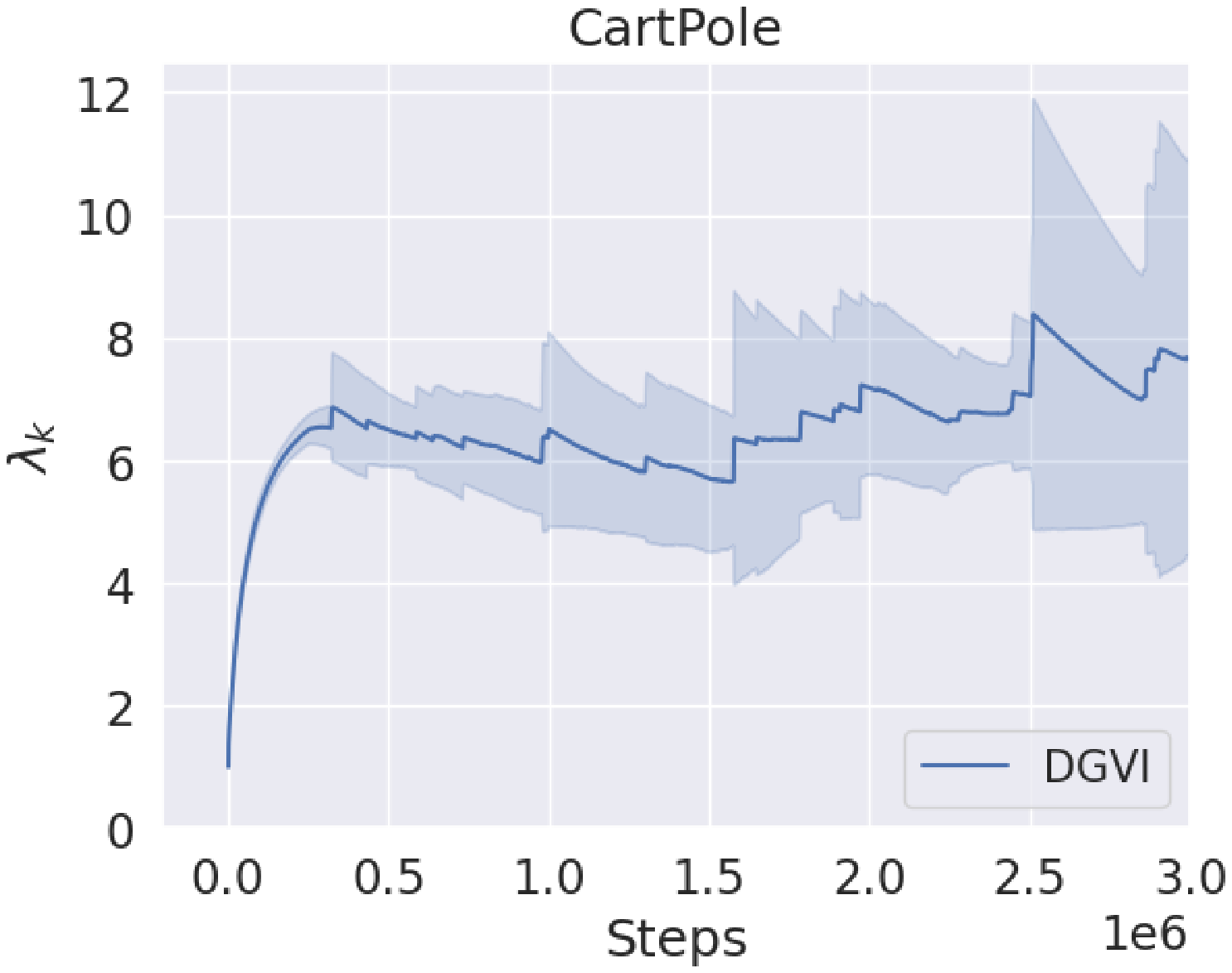}
    \end{minipage}
    \caption{(Top) Training curves on the discrete control tasks and (bottom) the KL coefficient in DGVI. The solid curves show the mean and the shaded regions show the standard deviation over the five independent trials.}
    \label{fig:return_comparison}
\end{figure}

\begin{figure}[t]
    \centering
    \begin{minipage}[c]{0.32\linewidth}
        \includegraphics[width=\linewidth]{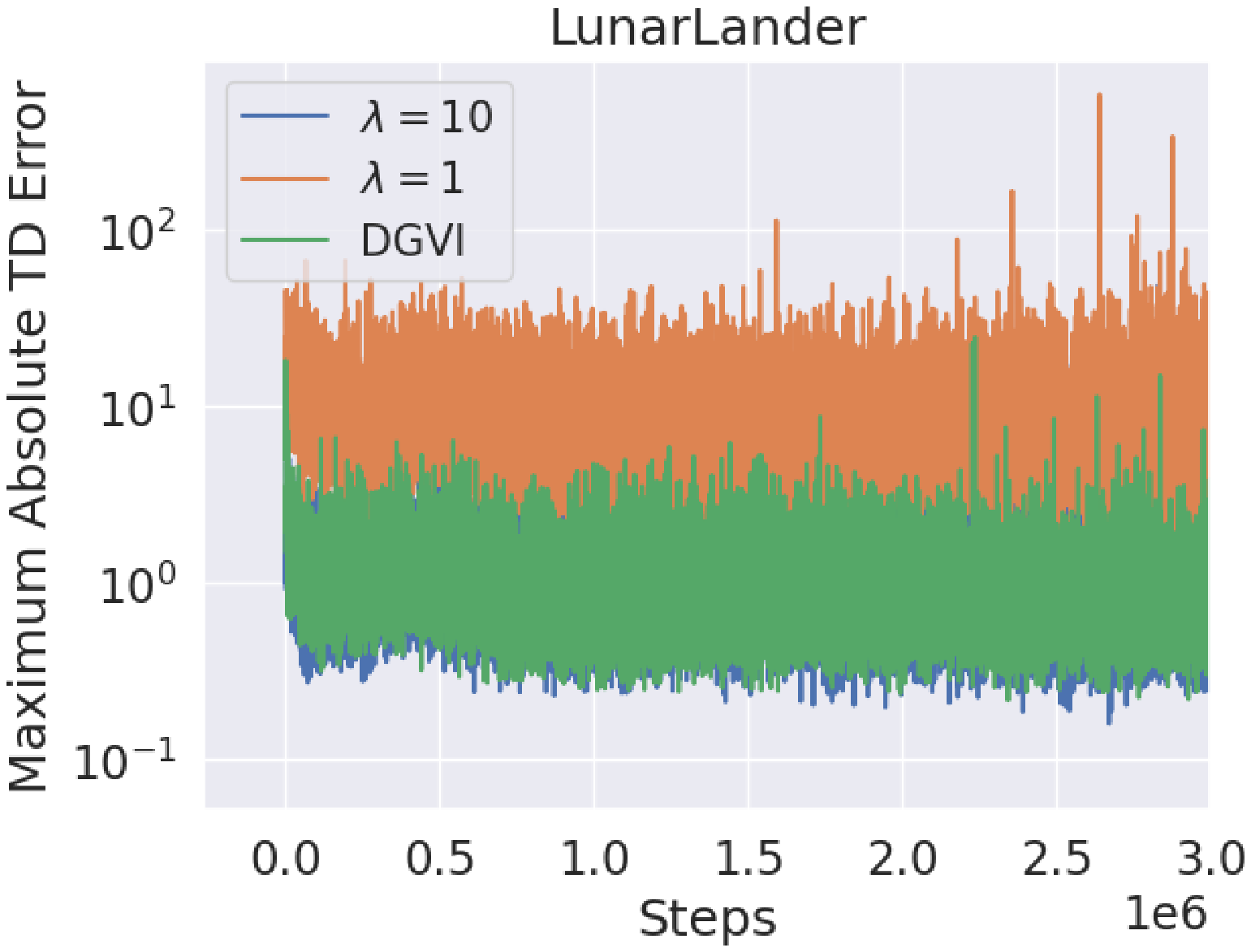}
    \end{minipage}
    \begin{minipage}[c]{0.32\linewidth}
        \includegraphics[width=\linewidth]{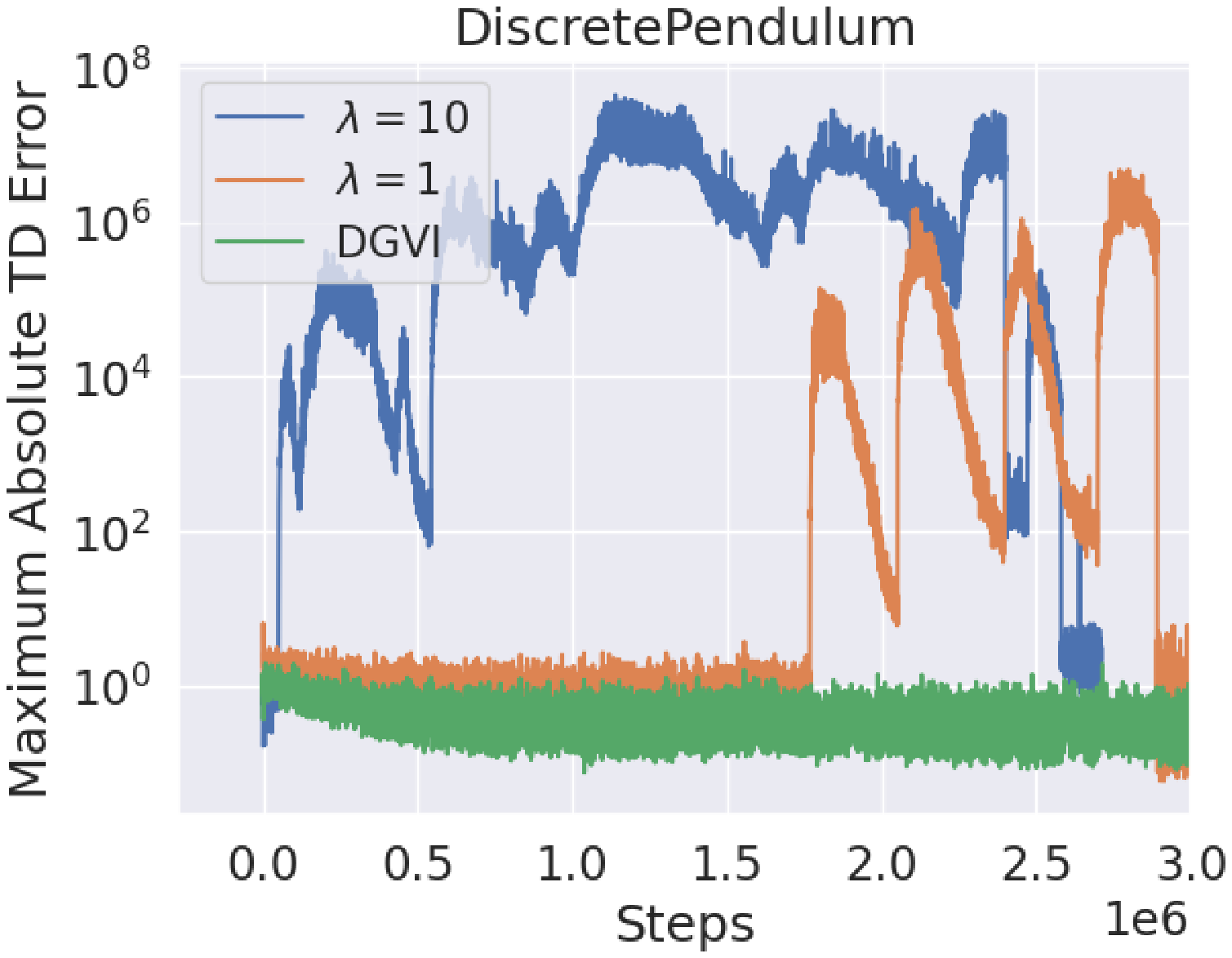}
    \end{minipage}
    \begin{minipage}[c]{0.32\linewidth}
        \includegraphics[width=\linewidth]{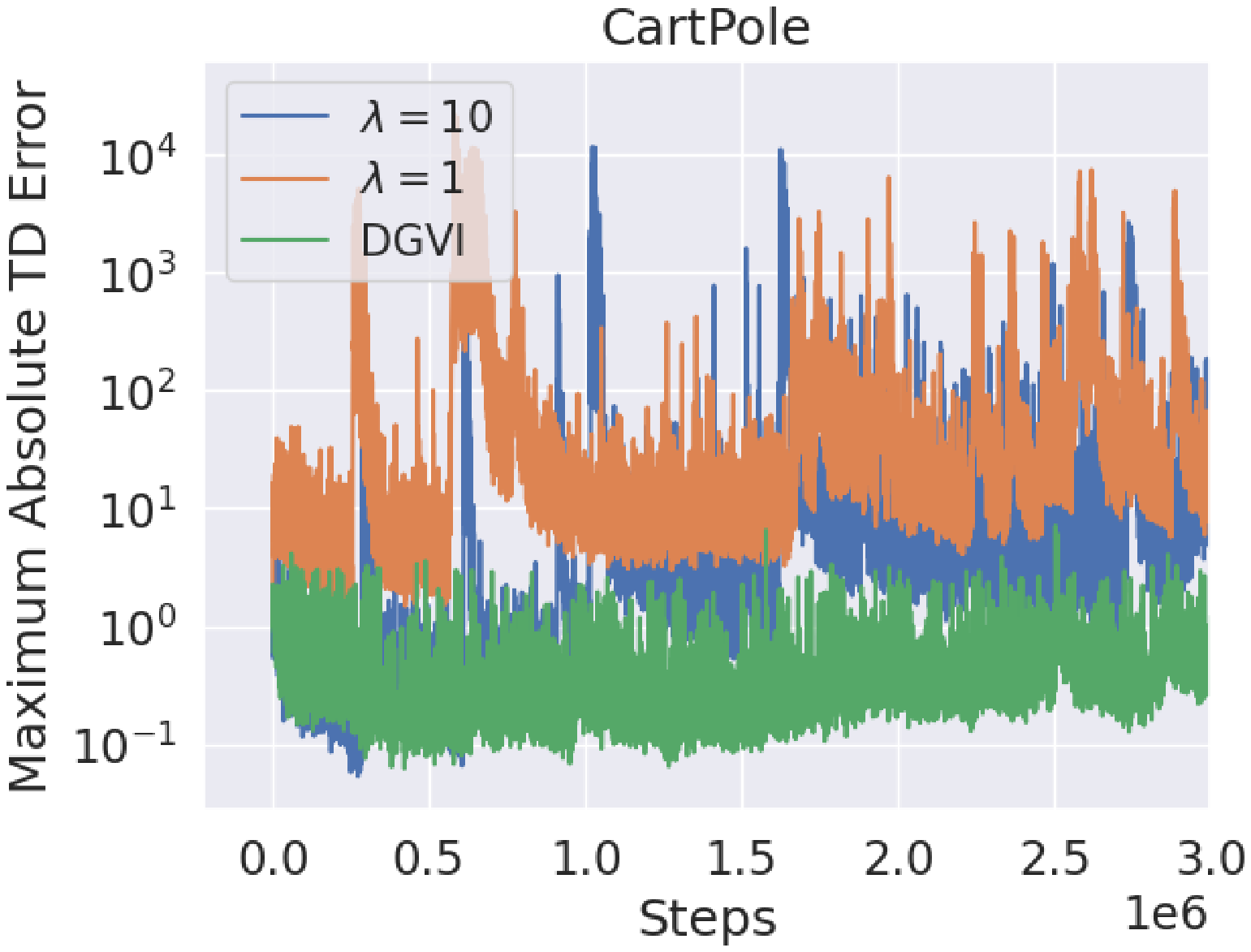}
    \end{minipage}
    \caption{The maximum absolute TD errors. The results are averaged over the five independent trials. DGVI shows much smaller TD errors in DiscretePendulum and CartPole.}
    \label{fig:td}
\end{figure}

\section{Related Work and Discussion}\label{sec:related_work}

The recent boom in the literature on KL-regularized ADP~\citep{azar2012dynamic,ghavamzadeh2011speedy,bellemare2016increasing,vieillard2020momentum,kozunoCVI} has demonstrated the effectiveness of KL regularization against estimation errors.
The most relevant algorithms to our proposed approach are Mirror Descent Value Iteration (MD-VI)~\citep{geist19-regularized} and Munchausen Value Iteration (MVI)~\citep{vieillard2020munchausen}; those algorithms introduce a KL penalty on both the greedy and the evaluation steps.
While some studies have focused on the error bounds of KL-regularized AVI~\citep{vieillard2020leverage}, how dynamic changes in the regularization coefficient affect performance has been left largely untouched.
To the best of our knowledge, this paper is the first work to provide the error bound of an AVI with dynamic KL regularization.

While it has not been discussed in the ADP literature, dynamic KL regularization has appeared in many deep RL algorithms.
Dynamic KL regularization is often introduced to restrict aggressive policy improvement steps.
Trust region policy optimization (TRPO)~\citep{schulman2015trust} is one such seminal algorithm that introduces KL constraints to approximately ensure monotonic improvement.
Based on TRPO, many algorithms leverage the KL constraint and demonstrate promising performance on challenging environments~\citep{schulman1707proximal,nachum2017trust,abdolmaleki2018maximum}, and \citet{nachum2017trust} introduced a dynamic KL coefficient design to create a trust region.
However, the above-mentioned algorithms design the dynamic coefficient based on heuristics, while we design it by leveraging rigorous analysis as shown in Theorem~\ref{thm:GVI}.
Furthermore, trust-region methods consider the KL constraints even when there are no estimation errors, and thus they may overly slow down learning.

In addition to the dynamic KL regularization, DGVI has an important feature: error awareness.
One of the most well-known algorithms making use of TD error is Prioritized Experience Replay (PER)~\citep{schaul2015prioritized}.
PER utilizes TD error for prioritizing the samples in the replay buffer to increase the appearance of rare samples.
On the other hand, DGVI mitigates the effect of rare samples that may have huge TD errors by scaling its bootstrapping.
Thus, slower learning will be expected when exploration matters: the rare samples will have less of an affect than usual in DGVI.
We do not consider this problem as exploration that is out of our scope.

In this work, we do not consider Shannon entropy for regularization.
Some entropy regularized ADP literature has established that by augmenting the reward with Shannon entropy, the optimal policy becomes multi-modal and hence robust against adversarial settings~\citep{haarnoja2017reinforcement,pmlr-v80-haarnoja18b,ahmed2019understanding}.
We leave GVI with Shannon entropy regularization as future work due to the complex theoretical analysis.

\section{Conclusion}\label{sec:conclusion}

We have presented the first {\it error-aware} KL coefficient design for RL algorithms and developed a novel {\it error-aware} RL algorithm, Geometric Value Iteration (GVI), which features a dynamic \emph{error-aware} KL coefficient design aimed at mitigating the impact of errors on performance.
The theoretical error bound analysis provides two guidelines for efficient learning: The coefficient should be increased when a large error is induced but its effect should not be overly large.
This dynamic regularization allows GVI to address the trade-off problem between robustness and convergence speed, which has been largely left untouched in previous ADP studies.

In addition to GVI as an ADP scheme, we further combined GVI with deep networks.
Based on the recent framework introduced by \citet{vieillard2020munchausen}, we implement GVI as a deep RL algorithm, and the resulting algorithm, deep GVI (DGVI), achieves robustness against errors by reducing the bootstrapping effect when it meets huge TD errors.
Our experiments verified not only the faster and more stable learning of GVI but also the more robust learning of DGVI even without target networks.

While our algorithm can be easily applied to standard deep RL frameworks, our empirical studies are limited to classic control tasks due to the expensive computational cost of recent Deep RL benchmarks, e.g., Atari games~\citep{bellemare2013arcade}.
We believe that the classic control tasks are sufficient to verify our algorithms and thus leave evaluation on a set of high-dimensional benchmarks as future work.

\acks{
This work is partly supported by JSPS KAKENHI Grant Number 21H03522 and 21J15633.
}

\bibliography{acml21}

\clearpage
% cut here for main text and appendix

\appendix

\section{Proofs on the performance bound}\label{apdx:gpi-proof}
For the following proof, we define the greedy policy and the Bellman operator regularized by Shannon entropy as well as KL divergence as $\gc_\mu^{\lambda, \tau}(q) = \argmax_{\pi\in\Delta^\s_\A} \left(\langle \pi, q\rangle - \lambda \kl(\pi||\mu) + \tau \hc(\pi) \right)$ and $T_{\pi\sep\mu}^{\lambda,\tau} q = r + \gamma P \left(\langle \pi, q\rangle - \lambda \kl(\pi||\mu)\right) + \tau \hc(\pi)$, respectively.
We also note the following fact about the greedy policy~\citep{vieillard2020leverage}:
\begin{equation}\label{eq:analytical}
\gc_\mu^{\lambda, \tau}(q) = \argmax_{\pi\in\Delta^\s_\A} \left(\langle \pi, q\rangle - \lambda \kl(\pi||\mu) + \tau \hc(\pi) \right) \propto \mu^{\frac{\lambda}{\lambda + \tau}} \exp \frac{1}{\lambda + \tau}q , 
\end{equation}
and we have the following maximum:
\begin{equation}\label{eq:maximum}
\max_{\pi\in\Delta^\s_\A} \left(\langle \pi, q\rangle - \lambda \kl(\pi||\mu) + \tau \hc(\pi) \right) = (\lambda + \tau) \ln\langle \un, \mu^{\frac{\lambda}{\lambda + \tau}}\exp \frac{q}{\lambda + \tau}\rangle .
\end{equation}
Before going to the proof of Theorem~\ref{thm:GVI}, we provide the following proposition.

\begin{proposition}
    \label{thm:equivalence}
    Define $Z_k=\sum^{k}_{j=0}\eta_j$, $h_0=q_0$, and $h_k$ for $k\geq 1$ as the average of past smoothed $q$-functions: $h_{k} = \frac{1}{Z_k}\sum^k_{j=0} \eta_j q_j = \frac{Z_{k-1}}{Z_{k}} h_{k-1} + \frac{\eta_k}{Z_{k}} q_{k}$.
    If $\lambda_k>0$ for all $k$, GVI is equivalent to the following iteration:
\begin{equation}
    \begin{cases}
        \pi_{k+1} = \gc^{0,{\frac{1}{Z_k}}}(h_k) \\
        q_{k+1} = (T_{\pi_{k+1}\sep\pi_k}^{{\frac{1}{\eta_k}},0})^m q_k + \epsilon_{k+1} \\
        h_{k+1} = \frac{1}{Z_{k+1}}\sum^{k+1}_{j=0} \eta_j q_j = {\frac{Z_{k}}{Z_{k+1}}} h_k + {\frac{\eta_{k+1}}{Z_{k+1}}} q_{k+1}.
    \end{cases}.\label{eq:da-gpi}
\end{equation}
\end{proposition}

\begin{proof}
Using Eq.~\eqref{eq:analytical} and by direct induction, we have $\pi_{k+1} \propto \pi_k \exp \eta_k q_k \propto \dots \propto \exp \sum_{j=0}^k \eta_j q_j= \exp Z_k h_k$.
Eq.~\eqref{eq:analytical} also provides $\argmax_{\pi\in\Delta_\A^\s}\left(\langle\pi, q\rangle + \tau \hc(\pi)\right) \propto \exp(\frac{1}{\tau}q)$. 
Hence, $\pi_{k+1}$ satisfies
$
    \pi_{k+1} = \argmax_{\pi\in\Delta_\A^\s}\left(\langle\pi, h_k\rangle + \frac{1}{Z_k} \hc(\pi)\right) = \gc^{0,\frac{1}{Z_k}}(h_k).
$
\end{proof}
We now prove the error-bound of GVI using Eq.~\eqref{eq:da-gpi}.
\paragraph{Proof.}

We first transform $q_* - q_{\pi_{k+1}}$, the difference between the optimal value function and the value function computed by Eq.~\eqref{eq:da-gpi}, using the following useful lemma:
\begin{lemma}[\citet{kakade-cpi}]
    \label{lemma:value_residual}
    For any $q\in\R^{\s\times\A}$ and $\pi\in\Delta_\A^\s$, we have
    $
        q_\pi - q = (I-\gamma P_\pi)^{-1}(T_\pi q - q).
    $
\end{lemma}
Using Lemma~\ref{lemma:value_residual}, $q_* - q_{\pi_{k+1}}$ can be transformed as
\begin{align}
    q_* - q_{\pi_{k+1}} &= q_* - h_k + h_k - q_{\pi_{k+1}} \nonumber
    \\
    &= (I-\gamma P_{\pi_*})^{-1}(T_{\pi_*} h_k - h_k) - (I-\gamma P_{\pi_{k+1}})^{-1}(T_{\pi_{k+1}}h_k - h_k). \label{eq:appx:davi10:decomposition}
\end{align}
Since the KL regularization vanishes after the iteration converges, the optimal policy must be deterministic, and hence $\hc(\pi_*)=0$.
Since $\pi_{k+1}$ is the regularized greedy policy, we have
\begin{align}
    \pi_{k+1} = \gc^{0,\frac{1}{Z_k}}(h_k) \nonumber
    &\Rightarrow
    \langle \pi_{k+1}, h_k \rangle + \frac{1}{Z_k} \hc(\pi_{k+1}) \geq \langle \pi_*, h_k\rangle + \frac{1}{Z_k}\hc(\pi_*) \nonumber
    \\
    &\Rightarrow
    T_{\pi_{k+1}}^{0,\frac{1}{Z_k}} h_k = T_{\pi_{k+1}} h_k + \gamma \frac{1}{Z_k} P\hc(\pi_{k+1}) \geq T_{\pi_*} h_k.
\end{align}
Using this with Eq.~\eqref{lemma:value_residual} and the fact that for any $\pi$ the matrix $(I-\gamma P_\pi)^{-1} = \sum_{t\geq 0} \gamma^t P_\pi^t$ is positive, we have the following inequality:
\begin{equation}
    q_* - q_{\pi_{k+1}} \leq (I-\gamma P_{\pi_*})^{-1}(T_{\pi_{k+1}}^{0,\frac{1}{Z_k}} h_k - h_k)
    - (I-\gamma P_{\pi_{k+1}})^{-1}(T_{\pi_{k+1}}^{0,\frac{1}{Z_k}}h_k - h_k - \gamma \frac{1}{Z_k} P\hc(\pi_{k+1})).
    \label{eq:appx:davi10:decomposition2}
\end{equation}
As for the residual $T_{\pi_{k+1}}^{0,\frac{1}{Z_k}} h_k - h_k$, we have the following useful lemma:

\begin{lemma}
    \label{lemma:q_to_h}
    For any $k\geq 1$, we have
    $
        \eta_k T_{\pi_{k+1}\sep\pi_k}^{\frac{1}{\eta_k},0} q_k = Z_k T_{\pi_{k+1}}^{0,\frac{1}{Z_k}} h_k - Z_{k-1} T_{\pi_k}^{0,\frac{1}{Z_{k-1}}} h_{k-1}.
    $
    For $k=0$, we have
    $
        \eta_0 T_{\pi_1\sep\pi_0}^{\frac{1}{\eta_0},0}q_0 = Z_0 T_{\pi_1}^{0,\frac{1}{\eta_0}} h_0 - \gamma P \hc(\pi_0).
    $
\end{lemma}
\begin{proof}
Using the definition of $\pi_k$ and $h_k$, the following equation holds.
    \begin{align}
        \eta_k q_k + \ln\pi_k = \eta_k q_k + (Z_{k-1} h_{k-1} - \ln\langle 1, \exp {Z_{k-1} h_{k-1}}\rangle)
        = Z_k h_k - \ln\langle 1, \exp {Z_{k-1} h_{k-1}}\rangle.
    \end{align}
Therefore, we have
$
        \langle \pi, \eta_k q_k\rangle - \kl(\pi||\pi_k) = \langle \pi, Z_k h_k\rangle - \langle \pi, \ln\pi\rangle - \ln\langle \un, \exp{Z_{k-1} h_{k-1}}\rangle.
$
From Eq.~\eqref{eq:maximum}, the maximum of $\langle \pi, Z_k h_k\rangle - \langle \pi, \ln\pi\rangle$ is $\ln \langle \un, \exp Z_k h_k\rangle$, and the maximizer is $\pi_{k+1}$ from the definition. 
By substituting $\pi_{k+1}$ to $\pi$, the following equation holds:
    \begin{equation}
        \langle \pi_{k+1}, \eta_k q_k\rangle - \kl(\pi_{k+1}||\pi_k) = Z_k \frac{1}{Z_k} \ln \langle \un, \exp{Z_k h_k}\rangle - Z_{k-1} \frac{1}{Z_{k-1}} \ln\langle \un, \exp{Z_{k-1} h_{k-1}}\rangle.
    \end{equation}
From Eq.~\eqref{eq:maximum}, $\frac{1}{Z_k} \ln \langle 1, \exp{Z_k h_k}\rangle$ is the maximum of $\langle\pi, h_k\rangle + \frac{1}{Z_k} \hc(\pi)$, and the associated maximizer is again $\pi_{k+1}$. 
    Hence, the following equation holds:
    \begin{align}
        \langle \pi_{k+1}, \eta_k q_k\rangle &- \kl(\pi_{k+1}||\pi_k)= Z_k\left(\langle\pi_{k+1}, h_k\rangle + \frac{1}{Z_k}\hc(\pi_{k+1})\right) - Z_{k-1} \left(\langle\pi_{k}, h_{k-1}\rangle + \frac{1}{Z_{k-1}}\hc(\pi_{k})\right). 
    \end{align}
    Observing that $\eta_k r = Z_k r - Z_{k-1} r$, we have the first part of the result:
    $
        \eta_k T_{\pi_{k+1}\sep\pi_k}^{\frac{1}{\eta_k},0} q_k = Z_k T_{\pi_{k+1}}^{0,\frac{1}{Z_k}} h_k - Z_{k-1} T_{\pi_k}^{0,\frac{1}{Z_{k-1}}} h_{k-1}.
    $
    For $k=0$, using the fact that $h_0 = q_0$, 
    \begin{equation}
        \eta_0 T_{\pi_1\sep\pi_0}^{\frac{1}{\eta_0},0} q_0 = \eta_0 r + \gamma P (\langle \pi_1, \eta_0 h_0\rangle + \eta_0 \frac{1}{\eta_0} \hc(\pi_1) + \eta_0 \frac{1}{\eta_0} \langle\pi_1, \ln\pi_0\rangle) = \eta_0 T_{\pi_1}^{0,\frac{1}{\eta_0}}h_0 - \gamma P \hc(\pi_0),
    \end{equation}
    where we use in the last line the fact that $\pi_0$, being uniform,
    $
        \langle\pi_1, \ln\pi_0\rangle = - \ln |\A| = -\hc(\pi_0).
    $
    This concludes the proof.
\end{proof}
Using Lemma~\ref{lemma:q_to_h}, we can provide induction on $h_k$.
\begin{lemma}
    \label{lemma:h_bellman_tau_zero}
    Define $E_k = -\sum_{j=1}^k \eta_j \epsilon_j$ and $X_k=\sum_{j=0}^k (\eta_{j+1} - \eta_j) T^{\frac{1}{\eta_j},0}_{\pi_{j+1}|\pi_j}q_j$.
    For any $k\geq 1$, we have
    $
        h_{k+1} = \frac{Z_k}{Z_{k+1}} T_{\pi_{k+1}}^{0,\frac{1}{Z_k}} h_k
        + \frac{1}{Z_{k+1}} \left(\eta_0 q_0 - E_{k+1} + X_k -\gamma P\hc(\pi_0)\right).
    $
\end{lemma}
\begin{proof}
    Using the definition of $h_k$, Lemma~\ref{lemma:q_to_h}, and the fact that $q_{k+1} = T_{\pi_{k+1}\sep\pi_k}^{\frac{1}{\eta_k},0} q_k + \epsilon_{k+1}$, we have
    \begin{align}
        &Z_{k+1} h_{k+1} = \sum_{j=0}^{k+1} \eta_j q_j
        \nonumber = \eta_0 q_0 + \eta_1 q_1 + \sum_{j=1}^k \eta_{j+1} q_{j+1} \nonumber \\
        &= \eta_0 q_0 + \left((\eta_1-\eta_0)+\eta_0\right) T_{\pi_1\sep\pi_0}^{\frac{1}{\eta_0},0} q_0 + \eta_1 \epsilon_1 + \sum_{j=1}^k \left(\left((\eta_{j+1} - \eta_j) + \eta_j\right) T_{\pi_{j+1}}^{\frac{1}{\eta_j},0} q_j + \eta_{j+1} \epsilon_{j+1}\right) \nonumber\\
        &= \eta_0 q_0 + \left(Z_0 T_{\pi_1}^{0,\frac{1}{\eta_0}}h_0 - \gamma P \hc(\pi_0)\right) + \sum_{j=1}^k\left(Z_j T_{\pi_{j+1}}^{0,\frac{1}{Z_j}} h_j - Z_{j-1} T_{\pi_{j}}^{0,\frac{1}{Z_{j-1}}} h_{j-1}\right) + X_k - E_{k+1} \nonumber
        \\
        &= \eta_0 q_0 +X_k - E_{k+1} -\gamma P\hc(\pi_0) + Z_k T_{\pi_{k+1}}^{0,\frac{1}{Z_k}} h_k 
        \\
        &\Leftrightarrow
        h_{k+1}  = \frac{Z_k}{Z_{k+1}} T_{\pi_{k+1}}^{0,\frac{1}{Z_k}} h_k + \frac{1}{Z_{k+1}} \left(\eta_0 q_0 - E_{k+1} + X_k -\gamma P\hc(\pi_0)\right).
    \end{align}
\end{proof}

Using Lemma~\ref{lemma:h_bellman_tau_zero} and the fact that $Z_{k+1} h_{k+1} = Z_k h_k + \eta_{k+1} q_{k+1}$, we have
$
    T_{\pi_{k+1}}^{0,\frac{1}{Z_k}} h_k - h_k = \frac{1}{Z_k}\left(\eta_{k+1} q_{k+1} - \eta_0 q_0 + E_{k+1} - X_k + \gamma P \hc(\pi_0)\right). 
$
Injecting this last result into decomposition~\eqref{eq:appx:davi10:decomposition2}, we get
\begin{align}
    & q_* - q_{\pi_{k+1}} \leq (I-\gamma P_{\pi_*})^{-1}(T_{\pi_{k+1}}^{0,\frac{1}{Z_k}} h_k - h_k)
    - (I-\gamma P_{\pi_{k+1}})^{-1}(T_{\pi_{k+1}}^{0,\frac{1}{Z_k}}h_k - h_k - \gamma P\hc(\pi_{k+1}))\nonumber
    \\
    &\leq (I-\gamma P_{\pi_*})^{-1}\left(\frac{1}{Z_k}\left(Y_k + \gamma P \hc(\pi_0)\right)\right) - (I-\gamma P_{\pi_{k+1}})^{-1}\left(\frac{1}{Z_k}\left(Y_k - \gamma P\hc(\pi_{k+1})\right)\right),
\end{align}
where we write $Y_k=\eta_{k+1}q_{k+1} - \eta_0 q_0 + E_{k+1} - X_k$ for the uncluttered notation and the last inequality holds, since $-(I-\gamma P_{\pi_{k+1}})^{-1} P \hc(\pi_0) \leq 0$.
Next, using the fact that $q_* - q_{\pi_{k+1}}\geq 0$ and rearranging terms, we have
\begin{align}\label{eq:b4conclusion}
    q_* - q_{\pi_{k+1}} \leq &\left|\left((I-\gamma P_{\pi_*})^{-1}- (I-\gamma P_{\pi_{k+1}})^{-1}\right)\frac{E_{k+1}}{Z_k}\right|\nonumber
    \\
    &+ (I-\gamma P_{\pi_*})^{-1}\left|\frac{1}{Z_k}\left(\eta_{k+1} q_{k+1} - \eta_0 q_0 - X_k + \gamma P \hc(\pi_0)\right)\right|\nonumber \\
    &+ (I-\gamma P_{\pi_{k+1}})^{-1}\left| \frac{1}{Z_k}\left(\eta_{k+1} q_{k+1} - \eta_0 q_0 - X_k + \gamma P \hc(\pi_{k+1})\right)\right|.
\end{align}
From the assumptions $\|q_{k}\|_\infty\leq q_\text{max}$ for all $k$, we have
$
    \|X_k\|_\infty = \|\sum^k_{j=0} (\eta_{j+1} - \eta_j) T^{\frac{1}{\eta_j},0}_{\pi_{j+1}|\pi_j}q_j\|_\infty \leq q_{\text{max}} \sum^k_{j=0}|\eta_{j+1}-\eta_j|.
$
Combined with Eq.~\eqref{eq:b4conclusion}, we have
\begin{equation}
    \| q_* - q_{\pi_{k+1}}\|_\infty \leq \frac{2}{(1-\gamma)Z_k}\left(\left\|\sum_{j=1}^k \eta_j\epsilon_j\right\|_\infty + (\eta_{k+1} + \eta_0 + \sum^k_{j=0}\left|\eta_{j+1}-\eta_j\right|)q_\text{max}+\gamma\ln|A|\right).
\end{equation}

\section{Proof of Theorem~\ref{thm:GVI2}}\label{apdx:GVI2}

Define for any $k\geq 0$ the term  ${q'}_k=\lambda_{k+1}\left(q_k-\ln \pi_k\right)$.
By basic calculus, the evaluation step of Eq.~\ref{eq:GVI2} can be transformed as

\begin{align}
& \;\; q_{k+1} = \frac{r}{\lambda_{k+1}} + \ln \pi_{k+1}  + \frac{\lambda_k}{\lambda_{k+1}} \gamma P\left\langle\pi_{k+1}, q_{k}-\ln \pi_{k+1} \right\rangle \nonumber \\ 
\Leftrightarrow
& \;\; \lambda_{k+1}\left({q}_{k+1} - \ln \pi_{k+1}\right) =r + \gamma P\left\langle\pi_{k+1}, \lambda_k \left({q}_{k} -\ln \pi_k \right) \right\rangle - \lambda_{k} \left\langle \pi_{k+1}, \ln \pi_{k+1}-\ln \pi_{k} \right\rangle \nonumber \\
\Leftrightarrow
& \;\; {q'}_{k+1} = r + \gamma P\left\langle\pi_{k+1}, {q'}_{k} \right\rangle - \lambda_{k} \kl({\pi_{k+1}}\|{\pi_{k}}) .
\end{align}

For the greedy step, we have
\begin{align}
    \argmax_\pi \left\langle\pi, q_{k}\right\rangle + \hc(\pi) &\propto \exp \left(q_k\right) = \pi_k \exp\left(\frac{{q'}_k}{\lambda_k}\right)\nonumber\\
    &\propto \argmax_\pi \left\langle\pi, {q'}_{k}\right\rangle + \kl \left(\pi \| \pi_k \right) .
\end{align}

Therefore, we have shown that 
\begin{align}
    &\begin{cases}
    \pi_{k+1}=\argmax_{\pi\in\Delta^\s_\A}\left\langle\pi, q_{k}\right\rangle + \hc(\pi)\\
    q_{k+1}=  \ln \pi_{k+1} + \frac{r}{\lambda_{k+1}} + \frac{\lambda_k}{\lambda_{k+1}} \gamma P\left\langle\pi_{k+1}, q_{k}-\ln \pi_{k+1} \right\rangle
    \end{cases} \nonumber\\
    \Leftrightarrow
    &\begin{cases}
    \pi_{k+1}=\argmax_{\pi\in\Delta^\s_\A}\left\langle\pi, {q'}_{k}\right\rangle - \lambda_k\kl({\pi}\|{\pi_k})\\
    {q'}_{k+1}=r + \gamma P\left\langle\pi_{k+1}, {q'}_{k} - \lambda_{k} \kl({\pi_{k+1}}\|{\pi_{k}})\right\rangle
    \end{cases}.
\end{align}

\section{Hyperparameters} \label{apdx:hypers}
Table.~\ref{tab:shared_params_1} lists the hyperparameters used in the comparative evaluation in Section.~\ref{sec:experiments}.

\begin{table}[t]
\renewcommand{\arraystretch}{1.1}
\centering
\caption{Hyperparameters of algorithms in deep RL experiments}
\label{tab:shared_params_1}
\vspace{1mm}
\begin{tabular}{l l| l }
\toprule
\multicolumn{2}{l|}{Parameter} &  Value\\
\midrule
\multicolumn{2}{l|}{\it{Shared}}& \\
& optimizer &Adam \\
& learning rate & $10^{-4}$\\
& discount factor ($\gamma$) &  0.99\\
& replay buffer size & $10^6$\\
& number of hidden layers  & 2\\
& number of hidden units per layer & 256\\
& number of samples per minibatch & 32\\
& activations & ReLU\\
\bottomrule
\end{tabular}
\end{table}

\section{Maze Environment Details}\label{apdx:maze}

\begin{wrapfigure}[12]{r}{0.3\textwidth}
  \begin{center}
    \vspace{-15pt}
    \includegraphics[width=0.2\textwidth]{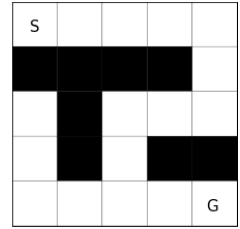}
    \vspace{-10pt}
    \caption{Example of a generated maze.}
  \end{center}\label{fig:maze-sample}
\end{wrapfigure}

For the tabular experiments, we use randomly generated $5\times5$ mazes.
Figure~\ref{fig:maze-sample} shows a sample maze used in the experiment.
The agent starts from a fixed position marked with {\it S} and can move to any of its neighboring states with success probability $0.9$, or to a different random direction with probability $0.1$.
The agent receives $+1$ reward when it reaches the goal marked with {\it G}, and the environment terminates after $25$ steps.
The agent cannot enter the black tiles.

\end{document}